\documentclass{article}

\usepackage[T1]{fontenc}
\usepackage[utf8]{inputenc}
\usepackage{hyperref}      
\usepackage{url}            
\usepackage{booktabs}       
\usepackage{amsfonts}       
\usepackage{nicefrac}       
\usepackage{microtype}      

\usepackage{graphicx}
\usepackage{amsmath,amssymb}
\usepackage{amsthm}
\usepackage{float}
\usepackage{caption}
\usepackage[ruled]{algorithm2e}
\usepackage{mathtools}
\usepackage{geometry}
\numberwithin{equation}{section}
\usepackage{graphicx}

\usepackage{authblk}

\usepackage{titlesec}
\titleformat*{\section}{\large\bfseries}
\titleformat*{\subsection}{\normalfont\bfseries}
      
\theoremstyle{plain}
\newtheorem{theorem}{Theorem}[section]
\newtheorem{lemma}[theorem]{Lemma}

\theoremstyle{definition}
\newtheorem{defn}{Definition}

\newtheorem{assumption}{Assumption}

\newcommand{\prob}{\mathbb{P}}
\newcommand{\real}{\mathbb{R}}
\newcommand{\xl}{\mathcal{X}^{(\ell)}}
\newcommand{\qil}{\mathcal{Q}_{-i}^{(\ell)}}
\newcommand{\xil}{X_{-i}^{(\ell)}}
\newcommand\numberthis{\addtocounter{equation}{1}\tag{\theequation}}

\newcommand{\back}{\backslash}

\newcommand{\mQ}{\mathcal{Q}}
\newcommand{\mT}{\mathcal{T}}
\newcommand{\mP}{\mathcal{P}}

\newcommand{\const}{c}

\DeclareMathOperator*{\argmax}{argmax}
\DeclareMathOperator{\conv}{conv}
\DeclareMathOperator{\diag}{diag}
\DeclareMathOperator{\sign}{sgn}
\DeclareMathOperator{\proj}{\Pi}

\title{Sparse Subspace Clustering with Missing and Corrupted Data}

\author[1]{Zachary Charles}
\author[2]{Amin Jalali}
\author[1,2]{Rebecca Willett}
\affil[1]{Department of Electrical and Computer Engineering, University of Wisconsin-Madison}
\affil[2]{Wisconsin Institute for Discovery, University of Wisconsin-Madison}

\usepackage{xcolor}

\begin{document}
\maketitle

\begin{abstract}
Given full or partial information about a collection of points that lie close to a union of several subspaces, {\em subspace clustering} refers to the process of clustering the points according to their subspace and identifying the subspaces. 
One popular approach, sparse subspace clustering (SSC) \cite{elhamifar2009sparse}, represents each sample as a weighted combination of the other samples, with weights of minimal $\ell_1$ norm, and then uses those learned weights to cluster the samples. SSC is stable in settings where each sample is contaminated by a relatively small amount of noise. However, when there is a significant amount of additive noise, or a considerable number of entries are missing, theoretical guarantees are scarce. 
In this paper, we study a robust variant of SSC and establish clustering guarantees in the presence of corrupted or missing data. We give explicit bounds on amount of noise and missing data that the algorithm can tolerate, both in deterministic settings and in a random generative model.
Notably, our approach provides guarantees for higher tolerance to noise and missing data than existing analyses for this method.  
By design, the results hold even when we do not know the locations of the missing data; e.g., as in presence-only data. 
\end{abstract}

\section{Introduction}

	In many applications, including image compression \cite{hong2006multiscale, yang2008unsupervised}, network estimation \cite{eriksson2011domain}, video segmentation \cite{costeira1998multibody, kanatani2001}, and recommender systems \cite{zhang2012guess}, what is ostensibly high-dimensional data can be modeled as data sampled from a {\it union of low-dimensional subspaces}. In~subspace clustering, we observe a data matrix $X \in \real^{n\times N}$ whose columns lie near a union of several low-dimensional subspaces of~$\real^n$. We wish to cluster the columns according to their subspace and infer the subspaces. In practice, $X$ may be corrupted by large amounts of noise, adversarial or otherwise, or it may have missing entries.

	Subspace clustering has experienced significant attention over the last decade, resulting in many different algorithms with varying levels of established theory. These include expectation-maximization methods~\cite{bradley2000k}, algebraic methods~\cite{vidal2005generalized}, matrix factorization methods~\cite{costeira1998multibody}, and local sampling methods \cite{rao2008motion}, among others. Of particular note is sparse subspace clustering (SSC)~\cite{elhamifar2009sparse}, which exhibits good empirical performance in many real data applications, most notably in computer vision applications, and enjoys provable guarantees on its performance~\cite{soltanolkotabi2012}. Moreover, various theoretical and empirical work has been carried out to show that SSC can handle outliers~\cite{soltanolkotabi2012} and certain forms of noisy measurements~\cite{soltanolkotabi2014robust,Wang2016}.

In the noiseless setup, SSC tries to represent each column of the data matrix $X$ as a linear combination of the other data points, with coefficients of minimal $\ell_1$ norm, by solving the following optimization program,
	\begin{equation}\label{eq:ssc}
	\min_{C}~\|C\|_{1} \quad \text{subject~to}~~~ X = XC,~\diag(C) = 0,\end{equation}
	where $\|C\|_1$ denotes the sum of absolute values of all entries in $C$. 
	The entries of the optimal solution are hoped to approximately encode whether pairs of columns of $X$ come from the same subspace. If so, spectral clustering can be applied to $|C|+|C|^T$ to recover a partition for columns of $X$. Most of theoretical analysis for SSC, e.g., see \cite{soltanolkotabi2012,Wang2016}, is dedicated to establishing guarantees on the pattern of zero entries in the optimal $C$ and presenting conditions under which $C_{ij}$ is zero whenever samples $i$ and $j$ do not belong to the same subspace. Such property, referred to as the {\em subspace detection property} in \cite{soltanolkotabi2012}, can then be used as a proxy for the performance of the spectral clustering phase and clustering. 
	
	In this paper we analyze the following {\em robust} variant of SSC, designed to work in the presence of corruptions and missing data,  
		\begin{equation}\label{lasso_ssc}
		\min_C~\|C\|_{1}+\frac{\lambda}{2}\|XC-X\|_F^2 \quad \text{subject~to}~~~ \diag(C) = 0,
		\end{equation}
		where $X$ is the observed matrix, either corrupted with additive noise or with zeros in place of missing entries. 
	We refer to this approach as LS-SSC since it involves a least-squares (LS) loss. This estimator has been previously considered in the literature with different theoretical guarantees \cite{soltanolkotabi2014robust,Wang2016} or with only empirical evidence for its effectiveness \cite{yang2015}. In this work, we consider different geometric quantities and provide an analysis that is similar in nature to \cite{Wang2016} but guarantees successful recovery for higher noise levels. Such improved thresholds directly translate to the ability to tolerate more missing entries: $O(n/d)$ versus $O(n/d^2)$ missing entries, for $d$-dimensional subspaces in $\real^n$.

	\subsection{Summary of Main Results}

		We provide analysis for (\ref{lasso_ssc}) and give theoretical guarantees for establishing the subspace detection property (see Definition~\ref{def:sub-det-prop}). We give a deterministic criteria for success based on the geometry of the true samples and the corruptions as well as the noise level defined as the maximum $\ell_2$ norm of the additive error in each observation.

The optimization problem in \eqref{lasso_ssc} has been previously studied in the literature \cite{Wang2016}. Our analysis in the deterministic setting follows the well-known strategy of constructing a dual certificate, similarly taken by \cite{Wang2016}. However, {\em our analysis removes a step of projection for the dual certificates} that is required in \cite{Wang2016}. As a result, we use a different definition of incoherence. This difference leads to an improved trade-off between the noise level and the dimensions of the subspaces. 

Our analysis allows us to show that \eqref{lasso_ssc} succeeds in the presence of noisy and missing data under the setup where subspaces are chosen uniformly at random and samples are drawn uniformly at random from these subspaces.
More specifically, we can tolerate corruptions with $\ell_2$ norm bounded by $O(1/\sqrt{d})$, or $O(n/d)$ missing entries per sample, where our subspaces are $d$-dimensional subspaces of $\real^n$. This is an improvement over what we get using results from \cite{Wang2016}: tolerating corruptions with $\ell_2$ norm bounded by $O(1/d)$, or $O(n/d^2)$ missing entries per sample. 
This is an important distinction in high-dimensional settings, where our results indicate the ability to recover from missing entries when the dimension of each subspace is as large as the ambient dimension, up to constants, while the results of \cite{Wang2016} require $d \lesssim \sqrt{n}$. Our improved rates concerning the number of missing entries are similar to the ones in a work that has been done in parallel \cite{tsakiris2018theoretical}. Our work uses different proof techniques that may extend to other
settings more easily, or at the very least provide different insights into subspace clustering with missing data.
		
		Finally, the presented framework is {\it location-agnostic}, as the estimator and the results need not know the location of the missing entries. This feature allows the estimator to be used in more general situations than algorithms such as those in \cite{yang2015,elhamifar2016high,tsakiris2018theoretical} that exploit the locations of the missing entries for estimation. This feature is important in certain real-data settings, such as the presence-only data described in Section~\ref{sec:compare}. 

	In the following, we state informal versions of our main results in both deterministic and random settings.

	\begin{theorem}[deterministic guarantee]
		Suppose we are given $X = Y+Z$ where $Y$ comes from a union of subspaces model and each column of $Z$ has $\ell_2$ norm bounded by an explicit function of the configuration of subspaces and the placement of the true samples on the subspaces. Then there is an explicit interval of $\lambda$ for which $\eqref{lasso_ssc}$ returns a nontrivial solution with no false positives.
	\end{theorem}

	\begin{theorem}[missing data]
		 Let $Y\in\real^{n\times N}$ be a matrix whose columns are drawn from the intersection of the unit ball and the union of $L$ $d$-dimensional subspaces that are drawn uniformly at random, and there are $\kappa d$ points in $Y$ corresponding to each subspace (i.e. $N = \kappa d L$).
		 Suppose we are given an incomplete version $X$ of $Y$ where zeros have been filled in to missing locations. If $d \leq O(n/\log N)$ and each column is missing at most $O(n/d)$ entries, then there is an explicit interval of $\lambda$ for which \eqref{lasso_ssc} returns a nontrivial solution with no false positives, with high probability. 
Here $O(\cdot)$ hides a small dependence on the number of points drawn from each subspace.
	\end{theorem}

\subsection{Comparison to Other Work}\label{sec:compare}
In this section, we briefly discuss the existing literature in relation to our main contributions.
\paragraph{Location-agnostic Estimation.}	
	The robust variant of SSC given in \eqref{lasso_ssc} does not require knowledge of the locations of the corruptions. This is in contrast with recent efforts in subspace clustering with missing data \cite{yang2015,elhamifar2016high,tsakiris2018theoretical} where the location of missing entries is a main ingredient for the estimator. 
		In practice, such information may not be available. For example, in some applications we face {\it presence-only} data where we only record the observed presence of a feature \cite{pearce2006modelling}. In ecological modeling, we often only have access to the observed population presence of a species in a given location but we do not know when a species is absent~\cite{ward2009presence}. Authors in~\cite{burnside2005knowledge} acknowledge that in structured mammography data, there is ambiguity in whether to interpret zeros as missing or as indicating that a breast imaging radiology feature is actually not present. This ambiguity is common in many models and applications \cite{liu2003building,elkan2008learning,fithian2013finite}, yet many proposed subspace clustering algorithms cannot be used in this setting.

	\paragraph{High-rank Matrix Completion.}
		When there are missing entries, subspace clustering can be viewed as a generalization of low-rank matrix completion \cite{candes2012exact}. Unlike low-rank matrix completion, the data matrix may have high rank if there are many subspaces. While there are many algorithms for low-rank matrix completion with theoretical guarantees on convergence and correctness (e.g. see \cite{candes2012exact} and references therein), such analysis has been more elusive for subspace clustering with missing data. Various algorithms for subspace clustering with missing data have been proposed \cite{gruber2004multibody, vidal2004motion, yang2015, pimentel2016group, elhamifar2016high,tsakiris2018theoretical,LADMC_allerton}. Many of these exhibit strong empirical results but lack theoretical guarantees. On the other hand, \cite{eriksson2012high} gives a theoretically justified method for subspace clustering with missing data, but requires an unrealistically large number of samples. Authors in \cite{pimentel2016information} give information-theoretic lower bounds on the number of observations per column required, but it is not known whether the aforementioned methods meet these bounds.

	\paragraph{Other Robust Variants.} 
	The key idea behind SSC is that if $X$ comes from a union of subspaces model, then the columns of $X$ are {\em self-expressive}. That is, each column of $X$ can be expressed as a linear combination of a small number of columns from the same subspace. Mathematically, this means that we can write $X = XC$ where $C$ is sparse and has 0 on the diagonals. If we assume that $X = Y+Z$ where $Y$ comes from a union of subspaces model, then $Y$ satisfies $Y = YC$ for C sparse. Therefore,
	we get $X-XC = Z-ZC = Z(I-C)$. Thus, minimizing a norm of $X-XC$ can be viewed as a proxy for minimizing a norm of $Z$, and prior information on the corruption mechanism can inform us on the design of the estimator. For example, when $Z$ is believed to be sparse, it makes sense to study another robust variant of SSC as 
		\begin{equation}\label{l1_ssc}
		\min_C~\|C\|_{1}+\lambda\|XC-X\|_1 \quad \text{subject~to}~~~ \diag(C) = 0.
		\end{equation}
	
	To analyze SSC and its variants, we often look at vectors $\nu$ that are optimal for the dual program to the SSC optimization. In studying \eqref{lasso_ssc}, the structure of the dual vectors allows us to understand them by analyzing their projection on to the subspaces in our union of subspaces model. However, in studying \eqref{l1_ssc}, this projection-based approach becomes more difficult due to the differing structure of the dual vectors.
		
		An earlier version of this manuscript was flawed in its analysis of \eqref{l1_ssc} due to an error in using the total inradius as opposed to the restricted inradius. We have removed the incorrect analysis of \eqref{l1_ssc} and updated our work to use restricted inradii accordingly when necessary (see Lemma \ref{polytope_rel_radii}, which replaces Lemma \ref{polytope_radii}).

\section{Problem Setup}

		We are given a matrix $X \in \real^{n\times N}$, where $X$ is the sum of an uncorrupted data matrix $Y$ and a noise matrix $Z$. We can view subspace clustering with missing data as a special case of this setup where the unobserved entries of $Y$ are replaced by 0 to obtain $X$. This is equivalent to the setting where the noise matrix $Z$ satisfies $Z_{ij} = -Y_{ij}$ for each missing entry $(i,j)$. We do not assume that we know the locations of the missing data. This way, we allow for a zero entry in $X$ to correspond to an actual zero or to a missing entry. We assume that the columns of $Y$ come from a union of $L$ subspaces
		$$S_1 \cup S_2 \cup \ldots \cup S_{L}.$$
		We make no assumptions on how the subspaces are aligned so that they can intersect arbitrarily.

		The original SSC method first solves the optimization problem in \eqref{eq:ssc}.
		SSC then uses spectral clustering \cite{ng2002spectral} on the affinity matrix $W = |C|+|C|^T$ to recover a partition of samples. 
		The key idea by \cite{elhamifar2009sparse} is that the columns of $X$ are {\it self-expressive}. If the $i$th column $x_i$ lies in a low-dimensional subspace from which we have enough closely-aligned samples, we can express $x_i$ as a sparse linear combination $c_i$ of other columns from that subspace. By enforcing $x_i = Xc_i$ and minimizing $\|c_i\|_1$, we hope to recover that sparse representation. If we approximately recover this $c_i$ for all $x_i$, then spectral clustering on the graph with edge weights from $W$ will recover the correct clusters.

		While this approach has empirical and theoretical guarantees when $X$ is uncorrupted by noise, we have to change our approach when $X$ has missing data or noise added. In order to perform sparse subspace clustering in these kinds of settings, we need to relax the assumption that $X = XC$. We do this by using a loss term of the form $\|X-XC\|_F^2$. This gives us the optimization problem LS-SSC in \eqref{lasso_ssc}. 
		For LS-SSC, \cite{Wang2016} describe a method to solve \eqref{lasso_ssc} using a modification of the ADMM method \cite{boyd2011distributed}. Moreover, \cite{soltanolkotabi2014robust} show that TFOCS \cite{becker2011templates} has competitive performance in solving this optimization program in practical applications. 
		We then apply spectral clustering to the weighted graph corresponding to the affinity matrix $W = |C|+|C|^T$. The resulting clusters are the subspace clusters we return. We use the standard technique of estimating the number of clusters $\hat{L}$ from the spectrum of the normalized Laplacian associated to $W$~\cite{von2007tutorial}. The full algorithm for LS-SSC is given in Algorithm~\ref{alg}.
		\begin{algorithm}
			{\bf Input: }A data matrix $X \in \real^{n\times N}$ and $\lambda > 0$.\\
		    1. Solve
		    $$\min_{C}~\|C\|_{1} + \frac{\lambda}{2}\|X-XC\|_F^2~~\text{s.t.}~~\diag(C)=0.$$\\
		    2. Form the weighted graph $G$ on $N$ vertices with affinity matrix $W = |C|+|C|^T$.\\
		    3. Let $\sigma_1\geq \sigma_2 \geq \ldots \geq \sigma_N$ denote the eigenvalues of the normalized Laplacian of $G$. Set
		    $$\hat{L} = N-\argmax_{i=1,\ldots, N-1} (\sigma_i-\sigma_{i+1}).$$
		    4. Apply spectral clustering to $G$ with $\hat{L}$ clusters.\\
			{\bf Output: }Clusters $\mathcal{X}_1,\ldots,\mathcal{X}_{\hat{L}}$.
			\caption{LS-SSC}
			\label{alg}
		\end{algorithm}

		We wish to find conditions for which LS-SSC has low clustering mismatch error. A good proxy for this is to show that $C$ has no false positives (\cite{von2007tutorial}), that is, for any $i$ and $j$ corresponding to different subspaces, $C_{ij} = 0$. This is reflected in the following definition.
		\begin{defn}[\cite{soltanolkotabi2012}] \label{def:sub-det-prop}
		We say that $X$ {\it obeys the subspace detection property with parameter~$\lambda$} if for all $\ell$ and for all $x_i$ corresponding to $S_\ell$, the columns $c_i$ of the solution to (\ref{lasso_ssc}) have non-zero entries corresponding only to columns in $Y$ sampled from $S_\ell$. We say that the {\it subspace detection property} holds if there is a non-empty interval of values of $\lambda$ for which the subspace detection property with parameter $\lambda$ holds.\end{defn}
		The above definition states that $c_i$ does not contain any entries corresponding to other subspace than the one to which $x_i$ corresponds. If the $c_i$ are non-zero and the subspace detection property holds, we should get low clustering error from spectral clustering.	

	\subsection{Preliminaries}

		Let $Y\in \real^{n\times N}$ be the matrix whose columns are all the uncorrupted samples. Let $N_\ell$ denote the number of columns in $Y$ drawn from $S_\ell$, and $Y^{(\ell)} \in \real^{n\times N_\ell}$ be the submatrix collecting such columns. We let $X^{(\ell)}$ and $Z^{(\ell)}$ denote the corresponding submatrices of $X$ and $Z$. Let $d_\ell$ be the dimension of $S_\ell$ and define $\kappa_\ell \coloneqq N_\ell/d_\ell$.

		For a matrix $A \in \real^{n\times m}$, we let $A_{-i}$ denote the $n \times (m-1)$ submatrix formed by removing the $i$th column. We let $\mathcal{Y} \subseteq \real^n$ denote the set of columns of $Y$ and let $\mathcal{Y}^{(\ell)}$ be the set of columns in $X$ corresponding to $S^{(\ell)}$. We define $\mathcal{X}$ and $\mathcal{X}^{(\ell)}$ analogously.

		For a matrix $A$, let $\mathcal{SC}(A)$ denote the symmetrized convex hull of its columns. If $A$ has columns $a_1,\ldots, a_n$, then $\mathcal{SC}(A)$ is $\conv(\pm a_1,\ldots, \pm a_n)$. We define $$\qil \coloneqq \mathcal{SC}(Y_{-i}^{(\ell)}).$$ Finally, we require some definitions from convex analysis. 		
		Given a set $\mathcal{P}\subseteq \real^d$, the {\it polar set} $\mathcal{P}^\circ$ of $\mathcal{P}$ is defined as
		$$\mathcal{P}^\circ = \{y \in \real^d :~ \langle x,y\rangle \leq 1\text{ for all }x \in \mathcal{P}\}.$$
		Note that $\mathcal{P}^\circ$ is a convex region. 

		\begin{defn}For any closed polytope $\mP$, we let $r(\mP)$ denote the {\it inradius} of $\mP$. This is defined as the radius of the largest Euclidean ball that can be inscribed in $\mP$.\end{defn}

		\begin{defn}For any closed polytope $\mathcal{P}$ and subspace $S$, we let $r_S(\mathcal{P})$ denote the {\it restricted inradius} of $\mathcal{P}$ with respect to $S$. This is defined as the radius of the largest disk in $S$ that can be inscribed in $\mathcal{P}$.\end{defn}

		Suppose $\mP \subseteq \real^n$ is symmetric and convex. Here symmetric means that $\mP = -\mP$. For such polytopes, the largest inscribed ball will necessarily be centered at 0. This follows from the fact that if a ball $B$ of radius $r$ can be inscribed in to $\mP$, then by symmetry, so can $-B$. Taking the convex hull of $B \cup -B$, we necessarily contain the ball of radius $r$ centered at 0. 
		Let $B_r$ denote the Euclidean ball centered at 0 of radius $r$ in $\real^n$. Then for $\mP$ symmetric and convex, we have 
		\begin{align*}
			r(\mP) = \sup\{ r : B_r \subseteq \mP\}~,~~
			r_S(\mP) = \sup \{ r : B_r\cap S \subseteq \mP\}.
		\end{align*}
		We now define the notion of circumradius.

		\begin{defn}For any closed polytope $\mP$, the {\it circumradius} of $\mP$, denoted $R(\mP)$ is the radius of the Euclidean ball containing $\mP$.\end{defn}	

		\begin{defn}For any closed polytope $\mP$, the {\it restricted circumradius} of $\mP$, denoted $R_S(\mP)$ is the radius of the smallest disk in $S$ containing $\mP$.\end{defn}
		With the same notation as above, and assuming $\mP$ is symmetric and convex, we have
		\begin{align*}
			R(\mP) = \inf\{ r : B_r \supseteq \mP\} ~,~~
			R_S(\mP) = \inf \{ r : B_r \cap S \supseteq \mP\}.
		\end{align*}
		For notational convenience, we define
		\begin{align}\label{def:r-ell}
			r_\ell \coloneqq \min_{i: x_i \in \xl}r_{S_\ell}(\qil) ~~,~~~
			r \coloneqq \min_{\ell=1,\ldots,L} r_\ell. 
		\end{align}
Finally, for a given subspace $S\subseteq \mathbb{R}^n$ and a point $x\in\mathbb{R}^n$, we denote by $\proj_S(x)$ the orthogonal projection of $x$ onto $S$. Given a matrix $A$, we overload the notation to denote by $\proj_S(A)$ the matrix which collects the projection of columns of $A$ onto $S$.

	\subsection{Dual Directions and Incoherence}\label{section:incoherence}

	    Given a vector $x$ and a matrix $A$, we define an optimization problem denoted $P(x,A,\lambda)$ as
	    \begin{equation}\label{eq:gen_lasso}
	    \min_{c,e}~\|c\|_1 + \frac{\lambda}{2}\|e\|_2^2~~\text{s.t.}~~e = x-Ac,\end{equation}
	    and its Lagrangian dual $D(x,A,\lambda)$ is given by
	    \begin{equation}\label{eq:dual_lasso}
	    \max_\nu~\langle x,\nu\rangle - \frac{1}{2\lambda}\|\nu\|_2^2~~\text{s.t.}~~\|A^T\nu\|_\infty \leq 1.\end{equation}	    

	    The optimization problem for LS-SSC in (\ref{lasso_ssc}) is equivalent to solving $P(x_i,X_{-i},\lambda)$ for $i = 1,\ldots, N$. We will use the dual program to analyze the geometry underlying our problem. For a given $x, A, \lambda$, let $\nu$ be the solution to $D(x,A,\lambda)$. If there are multiple solutions, select the one with the smallest $\ell_2$ norm. The corresponding {dual \em direction} $v$ is defined by
	     $$v(x,A,\lambda) = \nu/\|\nu\|_2.$$
		Define
		\begin{align}\label{def:incoh-us}
		v_i^{(\ell)} \coloneqq v(x_i^{(\ell)},\xil,\lambda) ~~\text{and}~~
		V^{(\ell)} \coloneqq [v_1^{(\ell)},\ldots, v_{N_\ell}^{(\ell)}].
		\end{align}
		We say that the set $\xl$ is {\it $\mu$-incoherent} with respect to the set $\mathcal{X}\back\mathcal{X}_{-i}^{(\ell)}$ if
		\begin{equation}\label{def:incoherence}
		\mu \geq \mu(\xl) \coloneqq \max_{y \in \mathcal{Y}\back\mathcal{Y}^{(\ell)}} \|(V^{(\ell)})^Ty\|_\infty
		= \max_{ \substack{y \in \mathcal{Y}\back\mathcal{Y}^{(\ell)} \\ 1 \leq i \leq N_\ell } } |\langle v_i^{(\ell)},y\rangle|
		\end{equation}
		For notational convenience, we consider
		\begin{align}\label{def:mu-ell}
			\mu_\ell \coloneqq \mu(\xl) ~~,~
		\mu \coloneqq \max_{\ell=1,\ldots,L} \mu_\ell.
		\end{align}
		This parameter $\mu$ is referred to as the {\em incoherence}. It is a measure of the alignment between the true observations $\mathcal{Y}^{(i)}$ from each subspace $S_i$ and the corrupted observations $\mathcal{X}^{(j)}$ from $S_j$ for $j \neq i$. Intuitively, the smaller $\mu$ is, the less aligned $\mathcal{Y}^{(i)}$ and $\mathcal{X}^{(j)}$ are. If $\mu$ is small enough, then it should be easier to group observations from distinct subspaces into distinct clusters. In Figure~\ref{fig}, we give a pictorial explanation of the subspace incoherence.

\begin{figure}
    \centering
    \begin{tabular}{cc}
        \includegraphics[width=.35\textwidth]{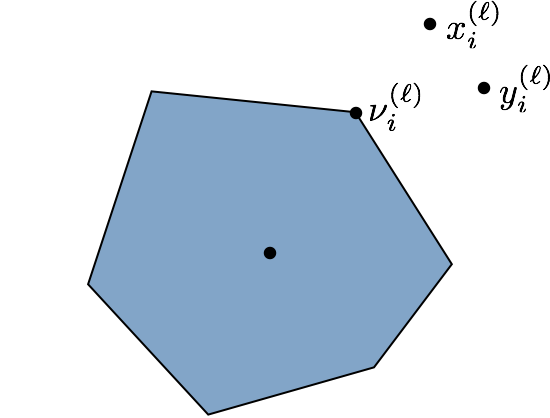}	&  \includegraphics[width=.5\textwidth]{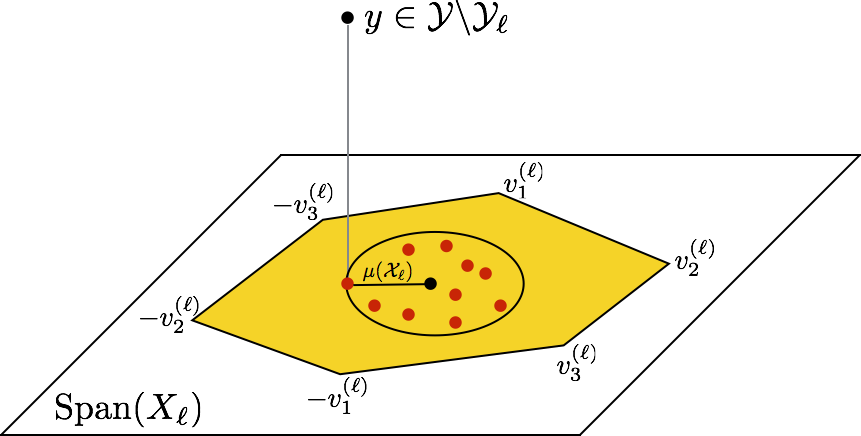}
    \end{tabular}
    \caption{Left: The dual direction $v(x_{i}^{(\ell)}, X_{-i}^{(\ell)},\lambda)$, where $x_{i}^{(\ell)}$ is a corrupted version of the true observation $y_{i}^{(\ell)}$. Right: The subspace incoherence $\mu_\ell$ is the radius of the smallest sphere in the span of $X^{(\ell)}$ containing all projections of $y \in \mathcal{Y}\backslash \mathcal{Y}^{(\ell)}$ onto the polytope determined by the dual directions.}
    \label{fig}
\end{figure}

This definition of subspace incoherence is a generalization of the subspace incoherence defined by \cite{soltanolkotabi2012} to the noisy setup described above. If there is no noise in the samples then for $\lambda$ sufficiently large these definitions of subspace incoherence will specialize to the definition in~\cite{soltanolkotabi2012}. 

On the other hand, our definition of incoherence is different from the one in \cite{Wang2016}, which is 
\begin{align}\label{def:incoh-WX}
v_i^{(\ell)} = \frac{\proj_{S_\ell}(\nu)}{\|\proj_{S_\ell}(\nu)\|_2} ~~~~ \text{(in \cite{Wang2016})}
\end{align}
where $\nu$ is the optimal solution to $D(x_i^{(\ell)},\xil,\lambda)$. This difference leads to different results for deterministic scenarios; i.e.,  Theorem 6 in \cite{Wang2016} and Theorem \ref{det_thm} below. 
Straightforward applications of both deterministic results to standard random generative models show that our results lead to improved tolerance to noise and missing entries.
We elaborate on these differences in Section~\ref{sec:compareWX}.

\section{Main Results}

	\subsection{Deterministic Model}		

		Let $Y$ be a matrix of samples drawn from a union of subspaces model. We assume that $Y$ is self-expressive, so that every column of $Y$ can be expressed as a linear combination of other columns from the same subspace. Let $Z$ be a deterministic noise matrix. We observe $X = Y+Z$. We assume that each column of $X$ has at least 1 non-zero entry. Define
		$$\delta = \max_i \|z_i\|_2.$$
		We will characterize how large $\delta$ can be for the subspace detection property to hold. For ease of analysis we assume, as in \cite{soltanolkotabi2012}, that each column $y$ of $Y$ lies on the unit sphere.\footnote{The results below all generalize to the case that $Y$ is unnormalized. The results will depend on the gap between the largest and smallest norm of columns in $Y$.}

		The following theorem gives conditions on the subspaces and noise under which LS-SSC will have the subspace detection property and produce a non-trivial output. Recall the definitions of~$r$ in~\eqref{def:r-ell} and $\mu$ in~\eqref{def:mu-ell}.

		\begin{theorem}[Deterministic model criteria]\label{det_thm}Suppose that
		\begin{equation}\label{det_crit}
		\delta < \dfrac{r-\mu}{5}\end{equation}
		and $\lambda$ lies in the non-empty interval
		\begin{equation}\label{lambda_crit}
		\dfrac{5}{2r + 3\mu} < \lambda < \frac{15}{2r+8\mu}.\end{equation}
		Then the subspace detection property with parameter $\lambda$ will hold. Moreover, we are guaranteed that any optimal $C$ for \eqref{lasso_ssc} have all nonzero columns. 
		\end{theorem}
		This theorem (proved in Appendix~\ref{app:pf_det_thm}) gives conditions that guarantee when $c_i$ will not have any false positives. It then refines this to find $\lambda$ for which the $c_i$ will also be non-trivial. 
When $\delta = 0$, the condition in \eqref{det_crit} reduces to having $\mu < r$, which is similar to the deterministic criteria in \cite{soltanolkotabi2012}.

		We will refer to the condition in (\ref{det_crit}) as the {\em geometric separation condition}. We will show that under random model assumptions, the geometric separation condition will hold with high probability.

	\subsection{Random Model}
		
		In the following, we define a random generative model for the uncorrupted points $Y$ on $L$ subspaces. 
		We then add a noise matrix $Z$ and impose Assumption \ref{indep_assumption} which comes after. 
		
		\begin{defn}[Random Model]\label{def:rand_model}
		Given a number $L\geq 1$, integer parameters $d_1,\ldots,d_L$, and positive parameters $\kappa_1,\ldots, \kappa_L$, the random generative model $\mathcal{RM}(L,\{d_\ell\}, \{\kappa_\ell\})$ is defined as the set of points generated as follows:  
		each subspace is drawn independently and uniformly at random among all $d_\ell$-dimensional subspaces, and $N_\ell = \kappa_\ell d_\ell$ samples are drawn independently and uniformly from the intersection of the subspace and the unit sphere. 
		\end{defn}

		\begin{assumption}\label{indep_assumption}
			For columns $y_i$ and $y_j$ drawn from distinct subspaces, $z_i$ is independent from $y_j$. In other words, the noise added to the points in one subspace is independent from the points in other subspaces. 
		\end{assumption}
		This assumption will appear in bounding the incoherence parameter. 
		As before, let $\delta = \max_{i} \|z_i\|_2$. The theorem below (proved in Appendix~\ref{app:pf_rand_thm}) gives conditions on the noise under which the subspace detection property holds and the columns of the optimal solution are non-trivial with high probability.

		\begin{theorem}[Random model criteria]\label{rand_thm}
Assume $N=\sum_{\ell=1}^L$ samples are drawn from $\mathcal{RM}(L,\{d_\ell\}, \{\kappa_\ell\})$ and stored in $Y$. Assume $Z$ is a noise matrix that satisfies Assumption~\ref{indep_assumption} with respect to~$Y$. 
		There are absolute constants $c_1, c_2$ such that, if for all $\ell$,
		\begin{equation}\label{rand_dl_cond}
		d_\ell < \dfrac{c_1\const(\kappa_\ell)^2\log(\kappa_\ell)}{\log N} n \end{equation}
		and
		\begin{equation}\label{rand_delta_cond}
		\delta < c_2\const(\kappa_\ell)\sqrt{\dfrac{\log(\kappa_\ell)}{d_\ell}},\end{equation}
		then with probability at least
		$$1-\frac{2}{N}-\sum_{\ell=1}^L N_\ell e^{-\sqrt{\kappa_\ell}d_\ell}$$
		the subspace detection property holds and the output of LS-SSC is non-trivial for all $\lambda$ satisfying
		$$\frac{5}{7}\sqrt{\dfrac{n}{6 \log N}} < \lambda < \frac{10}{3}\sqrt{\dfrac{n}{6 \log N}}.$$
		\end{theorem}
		
	Here, $\const(\kappa)$ is a constant depending only on $\kappa$. It is the same constant $\const(\kappa)$ as that in Section 1.4.2~of \cite{soltanolkotabi2012}. They also show that for all $\kappa$ sufficiently large we can take $\const(\kappa)$ to be a constant. Moreover, in many reasonable scenarios, we can treat $\const(\kappa_\ell)$ as a small constant.

		It is also worth noting that the condition in \eqref{rand_dl_cond} is, up to constants, the same as the condition on $d_\ell$ for noiseless subspace clustering \cite{soltanolkotabi2012}. Therefore, for $\delta$ sufficiently small, we recover the sufficient condition in \cite{soltanolkotabi2012} on $d_\ell$ for subspace detection under the same random model.

	\subsection{Missing Data}

		As in the random model, we assume that the $L$ subspaces are chosen independently and uniformly at random, and that the points $y_i$ are drawn randomly from the unit ball in $S_\ell$. Our data matrix $X$ satisfies $X = Y+Z$ where, $Z_{ij} = -Y_{ij}$ if we do not view $Y_{ij}$ and zero otherwise. In other words, $X$ is the entry-wise zero fill of $Y$ in the missing entries.

		We do not assume that we know the locations of the missing entries. A zero entry in $X$ could be because we do not observe that entry, or because there is a zero in $Y$ there. This allows LS-SSC to be used in more general settings such as with presence-only data.

Denote by $\odot$ the Hadamard (entry-wise) product. 
		\begin{assumption}\label{missing_assump}
		Given uncorrupted samples $Y\in\real^{n\times N}$ that are generated according to a random generative model, we observe $X = Y\odot \Omega$ where the mask matrix $\Omega\in\{0,1\}^{n\times N}$ is generated independently from $Y$. 
		\end{assumption}

		The following theorem (proved in Appendix~\ref{app:pf_scmd_rand}) gives conditions under which LS-SSC succeeds in this model with high probability.

		\begin{theorem}[Missing data criteria]\label{scmd_rand}
	    Assume $N=\sum_{\ell=1}^L$ samples are drawn from the random generative model $\mathcal{RM}(L,\{d_\ell\}, \{\kappa_\ell\})$ where 
		\begin{equation}\label{missing_d_cond}
		d_\ell < \dfrac{c_1\const(\kappa_\ell)^2\log(\kappa_\ell)}{\log N} n \end{equation}	    
	    Moreover, assume that the observed data matrix $X$ satisfies Assumption~\ref{missing_assump} and each column corresponding to $S_\ell$ is missing at most $M_\ell$ entries, where
	    $$ M_\ell \coloneqq c_3\const(\kappa_\ell)^2\log(\kappa_\ell)\dfrac{n}{d_\ell}. $$
	    Then, with probability at least
	    $$1-\frac{2}{N}-\sum_{\ell=1}^L N_\ell e^{-\sqrt{\kappa_\ell}d_\ell}-2\sum_{\ell=1}^L N_\ell e^{-M_\ell/16},$$
	    the subspace detection property holds and the output of LS-SSC is non-trivial for all $\lambda$ satisfying
		$$\frac{5}{7}\sqrt{\dfrac{n}{6 \log N}} < \lambda < \frac{10}{3}\sqrt{\dfrac{n}{6 \log N}}.$$\end{theorem}

	    Here, the constant $c_1$ is the same as in Theorem \ref{rand_thm}. This condition says that if the dimensions $d_\ell$ of our subspaces obey the same condition required for the noiseless SSC to succeed, then LS-SSC will succeed in the presence of $O(n/d)$ missing entries per column with high probability. If $d$ is constant with respect to $n$, then this says that LS-SSC can tolerate a constant fraction of missing entries in each column.

\subsection{Comparison with \cite{Wang2016}}\label{sec:compareWX}

Our alternative definition for the incoherence parameter is what led to the deterministic criteria in Lemmas \ref{intermediate_det} and \ref{lambda_lower} different from the ones in Section 5.3.3~of \cite{Wang2016}. Then, in combining these conditions for subspace detection property and non-triviality of the solutions, we get a different condition on the noise level $\delta$ in Theorem~\ref{det_thm}, which is
\begin{align}\label{eq:us}
\delta < \dfrac{r-\mu}{5},
\end{align}
than the one in Theorem 6 of \cite{Wang2016}, which is
\begin{align}\label{eq:WX}
\delta < \min_{\ell=1,\ldots, L} \frac{r(r_\ell-\mu_\ell)}{2+7r_\ell}.
\end{align}
While our definition of $\mu_\ell$ is different from \cite{Wang2016}, in the random model similar arguments can be used to bound both. The only property of the dual directions used in these arguments is the unit $\ell_2$ norm, which holds for both definitions of the dual direction (in \eqref{def:incoh-us} and \eqref{def:incoh-WX}). 
Under the assumption on the dimensions of the subspaces specified in \eqref{rand_dl_cond} and \eqref{missing_d_cond}, which is the same as the assumption in Theorem~11 in \cite{Wang2016} up to a constant factor, we can derive high-probability bounds on $r_\ell$ and $\mu_\ell$ (provided in \eqref{eq:r-ell-mu-ell} in Appendix \ref{app:pf_rand_thm}) as 
\begin{align*}
	    r_\ell &\geq \dfrac{\const(\kappa_\ell)\sqrt{\log(\kappa_\ell)}}{\sqrt{2d_\ell}}
	     \geq\sqrt{ \dfrac{24\log N }{n}}
	     \geq 2\mu_\ell.
\end{align*}
Plugging these bounds in \eqref{eq:us} yields
\begin{align}
\delta \lesssim \sqrt{\frac{1}{d_\ell}}
\end{align}
and allows for $M_\ell \simeq \frac{n}{d_\ell}$ missing entries per sample using this paper's analysis approach. 
Plugging the bounds in \eqref{eq:WX} yields
\begin{align}
\delta \lesssim \frac{1}{d_\ell}
\end{align}
which allows for $M_\ell \simeq \frac{n}{d_\ell^2}$ missing entries per sample using the \cite{Wang2016} analysis approach. 
This is an important distinction in high-dimensional settings, where our results indicate the ability of the method to recover from missing entries when the dimension of each subspace is as large as the ambient dimension, up to a constant factor, while a straightforward application of the results in \cite{Wang2016} require $d_\ell \lesssim \sqrt{n}$.

It is worth mentioning that in the noiseless setup of \cite{soltanolkotabi2012}, there are infinitely many solutions to the dual program and a projection step allows for well-posed analysis. While \cite{Wang2016} followed the same strategy of projection, it is not required as, in the noisy case and with \eqref{lasso_ssc}, the dual program is strongly concave and has a unique solution.

\section{Sketch of Proofs}\label{proof_section}

	\subsection{Dual Certificates and Admissible $\lambda$}

		Recall that in LS-SSC, we solve $P(x_i,X_{-i},\lambda)$ for $1 \leq i \leq N$. Fix some $i$ with $1 \leq i \leq N$. In order to guarantee that the solution $(c,e)$ to this problem contains no false positives, we consider an idealized problem. That is, we analyze
		$$P(x_i,X_{-i}^{(\ell)},\lambda).$$

		In other words, we attempt to solve LS-SSC but only using the matrix of other samples drawn from the same subspace. Since we do not know $X_{-i}^{(\ell)}$ a priori, we cannot solve this problem in practice. However, we show in Lemma \ref{lem:opt_cond} in the appendix that as long as the solution to this idealized problem and its dual satisfy certain conditions, the solution to $P(x_i, X_{-i},\lambda)$ will satisfy the subspace detection property. This follows from similar techniques to those in \cite{Wang2016}. We find the following sufficient condition for the subspace detection property.

		\begin{lemma}[establishing the subspace detection property]\label{intermediate_det}
		Suppose that for all $\ell$ we have
		\begin{equation}\label{inter_cond}
		2\lambda\delta < \dfrac{r_\ell-\mu_\ell-2\delta}{\mu_\ell+\delta}.\end{equation}
		Then the subspace detection property with parameter $\lambda$ holds.
		\end{lemma}

	    We now consider the parameter $\lambda$ in our optimization program. We want a condition on $\lambda$ that guarantees that the solution $(c,e)$ to $P(x_i,X_{-i},\lambda)$ is non-trivial. Using tools from convex geometry, one can show the following lemma. The full proof is contained in the appendix. Note that here we assume that $x_i \neq 0$. If $x_i = 0$, we have little to no hope of correctly clustering this point.

	    \begin{lemma}[non-trivial solution]\label{lambda_lower}If
	    $$\lambda > \frac{1}{r_\ell-2\delta-\delta^2},$$
	    then the solution $(c,e)$ to $P(x_i,X_{-i},\lambda)$ satisfies $c \neq 0$.\end{lemma}

	    We would like conditions on $\delta$ for which there is a non-empty set of $\lambda$ for which the subspace detection property holds and such that all of the columns of the output $C$ are non-zero. By Lemma \ref{lambda_lower} and Lemma \ref{intermediate_det}, we require that for all $\ell$,
	    \begin{gather*}
	    \dfrac{1}{r_\ell - 2\delta-\delta^2} < \lambda ~~~\text{and}~~~
	    2\lambda\delta < \dfrac{r_\ell-\mu_\ell-2\delta}{\mu_\ell+\delta}.\end{gather*}
	    Straightforward calculations show that for all $\delta$ satisfying (\ref{det_crit}) and for all $\lambda$ satisfying (\ref{lambda_crit}), the conditions of both Lemma \ref{intermediate_det} and \ref{lambda_lower} will hold. This allows us to derive Theorem \ref{det_thm}. The details are contained in the appendix.

	\subsection{Geometric Separation in the Random Model}

	    We now assume the conditions of the random model. By Theorem \ref{det_thm}, it suffices to give condition on $\delta$ and $d_\ell$ for which $\delta < \frac{r-\mu}{5}$ holds with high probability. Using techniques from \cite{soltanolkotabi2012} we lower bound $r$ and upper bound $\mu$, as 
	    	    \begin{gather*}
	    r_\ell \geq \dfrac{\const(\kappa_\ell)\sqrt{\log(\kappa_\ell)}}{\sqrt{2d_\ell}}
	    ~~,~~
	    \mu_\ell \leq \sqrt{ \dfrac{6\log N}{n}}.
	    \end{gather*}
with high probability. 

	\subsection{Subspace Clustering with Missing Data}

		As noted above, clustering with missing data is a special case of subspace clustering with additive noise. Let $X = Y+Z$ where each entry $Z_{ij}$ either equals $-Y_{ij}$ or 0. If the number of missing entries is not too large, then the corruption matrix $Z$ is relatively sparse.

		Recall that we assume that in each column coming from $S_\ell$, we have at most $M_\ell$ missing entries. We make no assumptions on how these missing entries are selected except that the missing locations are chosen independently from the observations; see Assumption~\ref{missing_assump}. We want a condition on $M_\ell$ under which LS-SSC succeed with high probability. It suffices to find a condition on $M_\ell$ such that the assumptions of Theorem \ref{rand_thm} hold.

		To do this, we have to control $\|z\|_2$ for each column $z$ of $Z$. In the missing data model, $z$ is the negative of the projection of a column $y$ of $Y$ on to $m$ coordinates. We can then use the following lemma.
		
		\begin{lemma}[Corollary~3.4~\cite{barvinok2005math}]\label{proj_lem}
		Let $x$ be uniformly drawn from the unit sphere and let $S \subseteq \real^n$ be any $m$-dimensional subspace. For $x \in S^{n-1}$, let $\proj_S(x)$ denote the orthogonal projection of $x$ on to $S$. Then for any $0 < \epsilon < 1$,
		\begin{align*}
			\prob \bigg( \|\proj_S(x)\|_2 \geq \frac{1}{1-\epsilon}\sqrt{ \frac{m}{n} } \bigg) \leq 2e^{-\epsilon^2m/4}.
		\end{align*}
		\end{lemma}

		This allows us to bound $\|z\|_2$ in terms of the number of missing entries $m$ and the ambient dimension $n$ of the data. If we select $m = O(n/d)$, where $d$ is the dimension of the subspace containing $y$, then with high probability the conditions of Theorem \ref{rand_thm} will hold. The details of the proof are left to the appendix.

\section{Conclusion}

	Subspace clustering in the presence of corrupted and missing data is an important task in a variety of machine learning problems, including high-rank matrix completion. Prior methods for subspace clustering with missing data either have few theoretical performance guarantees or are based on assumptions that are often not met in practice. Furthermore, prior methods often assume that the locations of missing entries are known a priori, whereas in practice we are often presented with a zero-filled matrix where we cannot distinguish between missing entries and zero-valued entries. 
	LS-SSC can be applied without knowing the locations of missing entries, making it especially applicable to presence-only data settings.
	Our study of LS-SSC in \eqref{lasso_ssc} addresses all these challenges. The theoretical guarantees we derive are based on straightforward assumptions on the data and yield bounds that match previous bounds when the noise level is small enough.

\bibliographystyle{alpha}

\bibliography{scmd_citations}

\newpage

\appendix

\section{Proofs for Deterministic Guarantees}

	\subsection{Optimality Conditions}

		Recall that given $x, A, \lambda$, we defined an optimization problem denoted $P(x,A,\lambda)$ as
	    \begin{equation*}
	    \min_{c,e}~\|c\|_1 + \frac{\lambda}{2}\|e\|_2^2 \qquad \text{subject~to}~~~e = x-Ac,\end{equation*}
	    and its Lagrangian dual $D(x,A,\lambda)$ is given by
	    \begin{equation*}
	    \max_\nu~\langle x,\nu\rangle - \frac{1}{2\lambda}\|\nu\|_2^2 \qquad \text{subject~to}~~~\|A^T\nu\|_\infty \leq 1.\end{equation*}	   

	    Fix $1 \leq i \leq N$ and suppose $y_i\in S_\ell$ for some $\ell\in\{1,\ldots,L\}$. We are interested in a characterization of the {\it support} of optimal solutions $c^*$ for $P(x_i,X_{-i},\lambda)$. If we can guarantee that the support of any such $c^*$ corresponds to only the columns of $X_{-i}$ coming from~$S_\ell$, then we will get the subspace detection property in Definition~\ref{def:sub-det-prop}. In order to guarantee this, we provide the following lemma, taken from \cite[Lemma~12]{Wang2016} and added here for posterity's sake. In a nutshell, the lemma ensures such a guarantee on the support of an optimal solution to $P$ whenever a certain {\em dual certificate of optimality} at $c^*$ exists for~$P$. Then, in Section~\ref{intermediate_det}, we show how to construct such dual certificate. The dual certificate strategy has been long used in compressed sensing and matrix completion literature, and has been employed in subspace clustering as well \cite{Wang2016}. 
	    \begin{lemma}\label{lem:opt_cond}
	    Let $A \in \real^{n\times N}, x \in \real^{n}$ be such that there are vectors $c,e,\nu$ and sets $S\subseteq T \subseteq \{1,\ldots, N\}$ satisfying $e = x-Ac$, $c$ has support $S$, and $\nu$ satisfies:
	    \begin{enumerate}
	    	\item $A_S^T\nu = \sign(c_S)$,
	    	\item $\nu = \lambda e$,
	    	\item $\|A_{T\cap S^c}^T\nu\|_\infty \leq 1$,
	    	\item $\|A_{T^c}^T\nu\|_\infty < 1$.
	    \end{enumerate}
	    Then any optimal solution $(c^*,e^*)$ to $P(x,A,\lambda)$ satisfies $c^*_{T^c} = 0$.
	    \end{lemma}
		\begin{proof}
	    Let $(c^*, e^*)$ be a solution to $P(x,A,\lambda)$. Then we have:
	    \begin{align*}
	    & \|c^*\|_1+\frac{\lambda}{2}\|e^*\|_2^2\\
	    &=\|c_S^*\|_1 + \|c_{T\cap S^c}^*\|_1 + \|c_{T^c}^*\|_1 + \frac{\lambda}{2}\|e^*\|_2^2\\
	    &\geq \|c_S\|_1 + \langle \sign(c_S),c_S^*-c_S\rangle + \|c_{T\cap S^c}^*\|_1 + \|c_{T^c}^*\|_1 + \frac{\lambda}{2}\|e^*\|_2^2.\numberthis \label{opt_calc1}\end{align*}
	    We now wish to find a lower bound for $\frac{\lambda}{2}\|e^*\|_2^2$ involving $e$. Note that function 
	    $$f(e) = \lambda (-\frac{1}{2}e^Te + e^Te^*),$$
	    for $\lambda >0$, has a unique maximum at $e^*$. Therefore, 
	    \begin{align*}\frac{\lambda}{2}\|e^*\|_2^2 &= f(e^*)\\
	    &\geq f(e)\\
	    &= \lambda (-\frac{1}{2}e^Te + e^Te^*)\\
	    &= \frac{\lambda}{2}\|e\|_2^2 + \langle \lambda e, e^*-e\rangle.\end{align*}
	    Using this fact and conditions 1 and 2 of the lemma on $\nu$ in \eqref{opt_calc1}, we have
	    \begin{align*}
		& \|c^*\|_1+\frac{\lambda}{2}\|e^*\|_2^2\\
		& \geq \|c_S\|_1 + \langle \nu, A_S(c_S^*-c_S)\rangle  + \|c_{T\cap S^c}^*\|_1 + \|c_{T^c}^*\|_1 + \frac{\lambda}{2}\|e\|_2^2 + \langle \nu, e^*-e\rangle\\
		&\geq \|c_S\|_1 + \frac{\lambda}{2}\|e\|_2^2 + \|c_{T\cap S^c}^*\|_1 - \langle \nu, A_{T\cap S^c}c_{T\cap S^c}^* \rangle + \|c_{T^c}^*\|_1 - \langle \nu, A_{T^c}c_{T^c}^*\rangle \\
		&\hspace{1cm}+ \langle \nu, A(c^*-c) + e^*-e\rangle.\numberthis \label{opt_calc2}\\
	    \end{align*}
	    Since $(c^*,e^*)$ and $(c,e)$ are feasible, we have $Ac^*+e^* = x = Ac+e$ which implies 
	    \begin{align*}
	    A(c^*-c) + e^*-e=0.\numberthis \label{opt_calc3}\end{align*}
	    Combining (\ref{opt_calc2}) and (\ref{opt_calc3}), and using the fact that $c$ has support $S$ so $c_S = c$, we have
	    \begin{align*}
	    &\|c^*\|_1 + \frac{\lambda}{2}\|e^*\|_2^2\\
	    &\geq \|c\|_1 + \frac{\lambda}{2}\|e\|_2^2 + \|c_{T\cap S^c}^*\|_1 - \langle \nu, A_{T\cap S^c}c_{T\cap S^c}^* \rangle + \|c_{T^c}^*\|_1 - \langle \nu, A_{T^c}c_{T^c}^*\rangle.\numberthis \label{opt-comp}.\end{align*}
	    By condition 3 of the lemma on $\nu$, we have
	    \begin{align*}
	    \langle\nu,A_{T\cap S^c}c^*_{T\cap S^c}\rangle &= \langle A_{T\cap S^c}^T\nu, c^*_{T\cap S^c}\rangle\\
	    & \leq \|A_{T\cap S^c}^T\nu\|_\infty\|c_{T\cap S^c}^*\|_1\\
	    &\leq \|c^*_{T\cap S^c}\|_1.\numberthis \label{opt-comp2}\end{align*}
	    By simple norm properties, we have
	    \begin{align*}
	    \|c^*_{T^c}\|_1 -\langle\nu,A_{T^c}c^*_{T^c}\rangle &= \|c^*_{T^c}\|_1 -\langle A_{T^c}^T\nu,c^*_{T^c}\rangle\\
	    &\geq \|c^*_{T^c}\|_1 -\|A_{T^c}^T\nu\|_\infty\|c^*_{T^c}\|_1\\
	    &\geq (1-\|A_{T^c}^T\nu\|_\infty)\|c^*_{T^c}\|_1 \numberthis\label{opt-comp3}
	    \end{align*}
	    Combining (\ref{opt-comp}), (\ref{opt-comp2}), and (\ref{opt-comp3}), we find
	    \begin{align*}
	    \|c^*\|_1 + \frac{\lambda}{2}\|e^*\|_1 \geq \|c\|_1 + \frac{\lambda}{2}\|e\|_1 + (1-\|A_{T^c}^T\nu\|_\infty)\|c^*_{T^c}\|_1.\end{align*}
	    By condition 4 of the lemma on $\nu$, we know that $(1-\|A_{T^c}^T\nu\|_\infty) > 0$. By optimality of $c^*, e^*$ for $P(x,A,\lambda)$, this implies that $\|c^*_{T^c}\|_1 = 0$ and so $c^*_{T^c} = 0$. Therefore, $c^*$ has support contained in $T$.\end{proof}

	\subsection{Dual Certificate Construction}
	\label{sec:dual-cert-construction}

	    We fix $i\in \{1,\ldots,N\}$, and assume $y_i\in S_\ell$ for some $\ell\in\{1,\ldots, L\}$. 
	    Let ${T \subseteq \{1,\ldots, N-1\}}$ denote the set of columns in $X_{-i}$ that correspond to $S_\ell$. Hence $|T|=N_\ell-1$. We wish to guarantee that for any solution $(c^*,e^*)$ to $P(x_i, X_{-i}, \lambda)$, the support of $c^*$ is contained in~$T$. By~Lemma~\ref{lem:opt_cond}, it suffices to exhibit, for each $i$, feasible vectors $c_i, e_i, \nu_i$ satisfying the conditions of this lemma when we take $A = X_{-i}$ and $x = x_i$. This is the step for constructing the dual certificate. 

		First, observe that as long as $A$ has at least one column, $P(x,A,\lambda)$ is trivially feasible. It is also trivially bounded by zero. Now, since the primal has a finite optimal value, then so does the dual, the optimal values coincide, and optimal solutions to both $P$ and $D$ exist. 
		
		In the following, we consider optimal solutions to $P(x_i,X_{-i}^{(\ell)},\lambda)$ and $D(x_i,X_{-i}^{(\ell)},\lambda)$, for~$\ell=1,\ldots,L$, and use them to construct candidates for which we verify conditions 1-4~of Lemma~\ref{lem:opt_cond}. While the first three conditions come naturally by construction, the last condition requires further assumptions on the incoherence measures (defined in Section \ref{section:incoherence}); e.g., the condition prescribed in the statement of Lemma \ref{intermediate_det}. Let us elaborate on this procedure. We would like to note that while establishing the first three conditions is similar to the material presented in Section~5.2~of \cite{Wang2016}, we believe Lemma~\ref{suff_det_cond} presents a simpler criteria for verifying Condition~4 and the subspace detection property. 
		
		\paragraph{Definitions.}
		Let $(c_i^{\ell},e_i^{\ell})$ be an optimal solution for $P(x_i,X_{-i}^{(\ell)},\lambda)$, 
		and $\nu_i^{\ell}$ be the optimal solution to $D(x_i,X_{-i}^{(\ell)},\lambda)$, 
		where $c_i^{\ell}\in\mathbb{R}^{N_\ell-1}$, $e_i^{\ell}\in\mathbb{R}^{n}$, and $\nu_i^{\ell}\in\mathbb{R}^{n}$. 
	    We define $c_i\in\mathbb{R}^{N-1}$ to be 0 in all columns outside of $X_{-i}^{(\ell)}$ and let it equal $c$ in columns corresponding to $X_{-i}^{(\ell)}$. In other words,  
	    \[
	    (c_i)_T = c ~\text{and}~(c_i)_{T^c} = 0 . 
	    \]
	    Let $S$ denote the support of $c_i$. 
	    Let $e_i = e_i^{\ell}$. Observe that $e_i = e_i^{\ell} = X_{-i}^{(\ell)}c_i^{\ell} - x_i = X_{-i} c_i - x_i $, so $(c_i,e_i)$ is a feasible point for $P(x_i,X_{-i},\lambda)$. Set $\nu_i \coloneqq \nu_i^{\ell}$. 
	    
	    \paragraph{Conditions 1-3.}
	    Complementary slackness implies
		\begin{align}\label{slack}
			(X_{-i}^{(\ell)})_S^T\nu_i = \sign((c_i)_S) ~~ \text{and} ~~
			\nu_i = \lambda e_i.
		\end{align}
		Therefore, $(c_i,e_i,\nu_i)$ satisfy conditions 1 and 2 of Lemma \ref{lem:opt_cond}. Moreover, observe that 
	    \[
	    \| (X_{-i})^T_{T\cap S^c}\nu_i\|_\infty 
	    \leq \|(X_{-i})_T^T\nu_i \|_\infty 
	    = \|(X_{-i}^{(\ell)})^T\nu_i\|_\infty 
	    \leq 1.
	    \]
	    Therefore, condition 3 of Lemma \ref{lem:opt_cond} also holds.     

\paragraph{Condition 4.}
With our definitions for $T$ and $S$ above, condition 4 of Lemma \ref{lem:opt_cond} is equivalent to the following: 
for each column $x = y+z$ of $X$ where $y \notin S_\ell$, namely, for all $x\in \mathcal{X}\back\mathcal{X}^\ell$, we have
	    \begin{equation*}
	    |\langle x,\nu_i \rangle|< 1.
	    \end{equation*}
	    Fix $x \in \mathcal{X}\back\mathcal{X}^\ell$, where $x = y+z$. Here, $y$ is the true vector and $z$ is the noise vector satisfying $\|z\|_2\leq \delta$. 
	    Observe that 
	    \begin{align*}
	    |\langle x,\nu_i\rangle | 
	    & \leq |\langle y,\nu_i\rangle| + |\langle z,\nu_i\rangle| \\
	    &= \|\nu_i\|_2\left|\langle y, \dfrac{\nu_i}{\|\nu_i\|_2}\rangle\right| + |\langle z,\nu_i\rangle| \\
	    &\leq \|\nu_i\|_2 \, \mu(\mathcal{X}^{(\ell)}) + \delta \|\nu_i\|_2 \\
	    & =(\mu(\xl)+\delta)\|\nu_i\|_2
	    \end{align*}
	    where in the last inequality, we used the definition of $\mu(\mathcal{X}^{(\ell)})$ in Section~\ref{section:incoherence} as well as the Cauchy-Schwarz inequality. 
	    Therefore, with the above constructions, the conditions of Lemma~\ref{lem:opt_cond} can be ensured using the following lemma. 

	    \begin{lemma}\label{suff_det_cond}
	    For any $i\in \{1,\ldots, N\}$, and $\ell\in\{1,\ldots, L\}$ for which $x_i\in S_\ell$, consider an optimal solution of $P(x_i,X_{-i}^{(\ell)},\lambda)$ as $(c_i^{\ell},e_i^{\ell})$ and the optimal solution of $D(x_i,X_{-i}^{(\ell)},\lambda)$ as $\nu_i$. If 
	    \begin{align}\label{eq:intermed-cond}
	    (\mu(\xl)+\delta)\|\nu_i\|_2 < 1,
	    \end{align}
	    then,
	    \[ 	    |\langle x,\nu_i \rangle|< 1 \]
	    for all $x\in \mathcal{X}\back\mathcal{X}^\ell$. 
	    \end{lemma}
		In other words, if \eqref{eq:intermed-cond} is satisfied, we have successfully constructed the dual certificate required in Lemma~\ref{lem:opt_cond}, which implies the subspace detection property. 
		Henceforth, we simplify the notation and use $\nu=\nu_i$. 
	    In the next section we will bound $\|\nu\|_2$ and use it, in conjunction with Lemma~\ref{lem:opt_cond} and Lemma~\ref{suff_det_cond}, to prove Lemma \ref{intermediate_det}.

	\subsection{Bounding $\|\nu\|_2$}

	    We now wish to bound $\|\nu\|_2$. We will do this by bounding the norm of its projection on to $S_\ell$ and the norm of its part that is orthogonal to $S_\ell$.

		Recall that the columns of $Y^{(\ell)}$ all lie in the subspace $S_\ell$. Let $\proj_{S_\ell}$ denote the projection operator on to $S_\ell$. We define
		\begin{align*}
		\nu_1 \coloneqq \proj_{S_\ell}(\nu) ~~ \text{and} ~~
		\nu_2 \coloneqq \nu-\nu_1.\end{align*}
		Observe that, by Moreau's decomposition theorem, $\nu_1$ and $\nu_2$ are orthogonal.
		We bound the norm of $\nu_1$ in Lemma~\ref{lem:nu1-norm} and the norm of $\nu_2$ in Lemma~\ref{lem:nu2-norm}, and combine the two results in Lemma~\ref{nu_norm}.

\subsubsection{Bounding $\|\nu_1\|_2$}
To bound $\|\nu_1\|_2$, we first establish a few auxiliary lemmas from convex geometry. This lemma establishes the same result as Lemma 16 in \cite{Wang2016} but corrects some errors in their proof. More specifically, their characterization of the boundary is not correct and ``restricted inradius'' should be used. 

		\begin{lemma}\label{perturb_rad}
		Consider two sets $\mathcal{Y} = \{y_1,\ldots,y_D\} \subseteq \real^d$ and $\mathcal{Z} = \{z_1,\ldots,z_D\} \subseteq \real^d$ satisfying
		$\|z_i\|_2 \leq \delta$, for $i=1,\ldots, D$, for some $\delta>0$. Define $x_i = y_i+z_i$, for $i=1,\ldots, D$, and $\mathcal{X} = \{x_1,\ldots,x_D\}$. 
		Let $\mQ$ be the symmetrized convex hull of $\mathcal{Y}$ and $\mT$ be the symmetrized convex hull of $\mathcal{X}$. Suppose that $r(\mQ) > \delta$. Then $r(\mT) \geq r(\mQ) - \delta$.
		\end{lemma}
		\begin{proof}
Consider a face $F$ of $\mQ$, the symmetrized convex hull of $\mathcal{Y}$. Without loss of generality, there are vertices $y_1, \ldots, y_k$, and values $s_1,\ldots,s_k \in \{-1,+1\}$, such that
		$$F = \left\{ \sum_{i=1}^k s_iy_iw_i ~\bigg|~ w_i \geq 0, \sum_{i=1}^k w_i = 1\right\}.$$
		Since $F$ is on the boundary of $\mQ$, every point $y \in F$ satisfies $\|y\|_2 \geq r(\mQ)$. Let $x_1,\ldots, x_k$ denote the points in $\mT$ corresponding to the $y_1,\ldots,y_k$. Define $F'$ by
		$$F' = \left\{ \sum_{i=1}^k s_ix_iw_i ~\bigg|~ w_i \geq 0, \sum_{i=1}^k w_i = 1\right\}.$$
		Now suppose $x \in F'$. Then we have
		\begin{align*}
		\|x\|_2 &= \left\|\sum_{i=1}^k s_ix_iw_i\right\|_2\\
			&= \left\|\sum_{i=1}^k s_iy_iw_i + s_iz_iw_i\right\|_2\\
			&\geq \left\|\sum_{i=1}^k s_iy_iw_i\right\|_2 - \sum_{i=1}^k w_i\|z_i\|_2\\
			&\geq r(\mQ)- \sum_{i=1}^k w_i\delta\\
			&\geq r(\mQ)-\delta.\end{align*}

		Let $C$ be the union of $F'$ over all faces $F$ of $\mQ$. Then, by the above argument, every point in $C$ has $\ell_2$ norm at least $r(\mQ)-\delta$. This implies that every point on the boundary of $\conv(C)$ has $\ell_2$ norm at least $r(\mQ)-\delta$; i.e., $$r(\conv(C)) \geq r(\mQ)-\delta.$$ 
		
		On the other hand, the symmetry in $\mQ$ provides pairs of symmetric faces (for the proper faces) $F$ and $-F$, hence pairs of symmetric sets $F'$ and $-F'$. This makes $C$ and $\conv(C)$ symmetric sets. Moreover, by construction, $x_i \in C$ for all $i=1,\ldots, D$. Therefore, $$\conv(C) \subseteq \mT$$ where $\mT$ is the symmetrized convex hull of $x_1,\ldots, x_D$. All in all, we get $$r(\mT) \geq r(\conv(C)) \geq r(\mQ) - \delta,$$ which establishes the claim of the lemma. 
		\end{proof}

	In the above lemma, if all points $x_i$, $y_i$ and $z_i$, for $i=1,\ldots,d$, live in a subspace $S$ of a bigger ambient space, a similar result would hold, where we adjust the definition of inradius to account for the degeneracy. In other words, in such case, assuming $r_S(\mQ)>\delta$, we can guarantee 
	\begin{align}
	r_S(\mT) \geq r_S(\mQ) - \delta. 
	\end{align}
	In our original setup, projecting noise vectors to the true subspace puts us in the above situation. 

		We will also use the following lemma about the relation between the restricted inradius and circumradius of a convex body. This lemma can be bound in \cite{brandenberg2005radii}.

		\begin{lemma}[\cite{brandenberg2005radii}]\label{polytope_radii}
		For any symmetric convex polytope $\mathcal{T}$,
		$$r(\mathcal{T})R(\mathcal{T}^\circ) = 1.$$\end{lemma}

		We can then use this to derive the following lemma about the restricted inradius of a symmetric convex polytope.
		\begin{lemma}\label{polytope_rel_radii}
		Suppose that $\mT$ is a symmetric convex polytope that lies in a subspace $S$. Then
		$$r_S(\mT)R_S(\mT^\circ\cap S) = 1.$$\end{lemma}
		\begin{proof}Suppose $\mT \subseteq \real^n$ and $S$ has dimension $d$. Without loss of generality, we assume that $S$ is the span of the elements $e_1,\ldots, e_d$ where $e_i$ denotes the $i$th vector of the standard basis. Let $\pi : \real^n \to \real^d$ be the map that sends a vector $(x_1,\ldots, x_n)$ to $(x_1,\ldots, x_d)$. Let $\mT' = \pi(\mT)$. Note that this is the same convex polytope as $\mT$, but in $\real^d$ instead of $\real^n$. In particular, $r_S(\mT) = r(\mT')$.

		Moreover, it is straightforward to see that $\pi(\mT^\circ) = \pi(\mT^\circ\cap S) = (\mT')^\circ$. In particular, $R_S(\mT^\circ\cap S) = R((\mT')^\circ)$. By Lemma \ref{polytope_radii}, we then have
		\begin{align*}
		r_S(\mT)R_S(\mT^\circ\cap S) &= r(\mT')R((\mT')^\circ) = 1.\qedhere\end{align*}\end{proof}

	    We can then show the following lemma, which is similar to the final result of Section 5.3.1~in \cite{Wang2016}, but uses the corrected version of their Lemma 16 proved as Lemma \ref{perturb_rad} above, as well as Lemma~\ref{polytope_rel_radii} instead of their Lemma 15. 
	    \begin{lemma}\label{lem:nu1-norm}
	    Under the same assumptions as in Lemma \ref{suff_det_cond}, as well as $r_{S_\ell}(\qil)>\delta$, we have
	    $$\|\nu_1\|_2 \leq \frac{1+\delta\|\nu_2\|_2}{r_{S_\ell}(\qil)-\delta}.$$\end{lemma}
		\begin{proof}

			Since $\nu$ is a feasible point of $D(x_i,X_{-i}^{(\ell)},\lambda)$, by the definition of $D$ we have
			$$\|(\xil)^T\nu\|_\infty \leq 1.$$
			For any column $x = y+z$ of $\xil$, we have $y\in S_\ell$, while $\nu_1\in S_\ell$ and $\nu_2 \in S_\ell^\perp$. Therefore, 
			\begin{align*}
			1 & \geq |\langle x,\nu \rangle |\\
			&\geq |\langle x,\nu_1 \rangle | - |\langle x,\nu_2 \rangle | \\
			&= |\langle \proj_{S_\ell}(x) ,\nu_1 \rangle | - |\langle z,\nu_2 \rangle | \\
			&\geq |\langle \proj_{S_\ell}(x),\nu_1\rangle| - \delta\|\nu_2\|_2
			\end{align*}
implying 
			\begin{align}\label{proj_bound}
			|\langle \proj_{S_\ell}(x),\nu_1\rangle| \leq 1 + \delta\|\nu_2\|_2.	
			\end{align}
			Let $M$ have columns given by $\proj_{S_\ell}(x)=y + \proj_{S_\ell}(z)$, for all $x\in \xil$, and let $\mT = \mathcal{SC}(M)$ be the symmetrized convex hull of the columns of $M$. Then (\ref{proj_bound}) implies that
			\begin{equation}\label{polar_proj}
			\dfrac{\nu_1}{1+\delta\|\nu_2\|_2} \in \mT^\circ.\end{equation}
			Since $\nu_1 \in S_\ell$, the circumradius of $\mT^\circ\cap S_\ell$ restricted to $S_\ell$ provides a bound on the norm of $\nu_1/(1+\delta\|\nu_2\|_2)$. Therefore,
			\begin{equation}
			\|\nu_1\|_2 \leq R_{S_\ell}(\mT^\circ\cap S_\ell)(1+\delta\|\nu_2\|_2).\end{equation}
			We now wish to bound $R_{S_\ell}(\mT^\circ\cap S_\ell)$. By Lemma \ref{polytope_rel_radii}, we have
			$$R_{S_\ell}(\mT^\circ\cap S_\ell) = \dfrac{1}{r_{S_\ell}(\mT)}.$$
			It now suffices to give a lower bound on $r_{S_\ell}(\mT)$. Recall that $\mT$ is the symmetrized convex hull of the columns of $Y_{-i}^{(\ell)}$, perturbed by the projection of columns of $Z_{-i}^{(\ell)}$ onto $S_\ell$. Since each column $z$ has norm at most $\delta$, its projection onto $S_\ell$ satisfies the same norm bound, and we can apply Lemma \ref{perturb_rad} to get
			$$r_{S_\ell}(\mT) \geq r_{S_\ell}(\mathcal{SC}(Y_{-i}^{(\ell)}))-\delta = r_{S_\ell}(\qil)-\delta.$$
			All in all,
			\[\|\nu_1\|_2 \leq \dfrac{1+\delta\|\nu_2\|_2}{r_{S_\ell}(\qil)-\delta}.\qedhere\] 
		\end{proof}
In bounding $\|\nu_1\|_2$, in Lemma \ref{perturb_rad} and Lemma \ref{lem:nu1-norm}, we have assumed $r_{S_\ell}(\qil)>\delta$. Since $r_\ell \coloneqq \min_{i: x_i \in \xl}r_{S_\ell}(\qil)$ and $r \coloneqq \min_{\ell=1,\ldots,L} r_\ell$, we have essentially assumed $r>\delta$. This is ensured by the conditions of Theorem \ref{det_thm}. 
\subsubsection{Bounding $\|\nu_2\|_2$}
		We bound $\|\nu_2\|_2$ in the following lemma. The result is similar to Section 5.3.2~of \cite{Wang2016}, but using the correct notions of restricted inradius and restricted circumradius discussed in Lemma~\ref{polytope_rel_radii}. 

		\begin{lemma}\label{lem:nu2-norm}
		Suppose every $y_i$, for $i=1,\ldots,N$, can be expressed as a linear combination of other points in the same subspace. In other words, suppose for any $i=1,\ldots,N$, and the corresponding $\ell\in\{1,\ldots, L\}$ for which $y_i\in S_\ell$, we have $y_i \in \operatorname{span}(Y_{-i}^\ell)$. 
		Then, 
		$$\|\nu_2\|_2 \leq \dfrac{\lambda\delta}{r_{S_\ell}(\qil)}+\lambda\delta.$$\end{lemma}

		\begin{proof}
		Recall that $\nu$ is a solution of $D(x_i,X_{-i}^{(\ell)},\lambda)$, and that by (\ref{slack}), we have $\nu = \lambda e_i$. Therefore, $\nu = \lambda(x_i-X_{-i}^{(\ell)}c_i)$. In the following, we will let $c_{ij}$ denote the $j$th entry of $c_i$. Therefore,
		\begin{align*}
		\|\nu_2\|_2 &= \|\proj_{S_\ell^\perp}(\nu)\|_2\\
		&= \lambda \|\proj_{S_\ell^\perp}(x_i-X_{-i}^{(\ell)}c_i)\|_2\\
		&= \lambda \|\proj_{S_\ell^\perp}(z_i) + \proj_{S_\ell^\perp}(Z_{-i}^{(\ell)}c_i)\|_2\\
		&\leq \lambda \|\proj_{S_\ell^\perp}(z_i)\|_2 + \lambda \| \proj_{S_\ell^\perp} (Z_{-i}^{(\ell)}c_i)\|_2\\
		&\leq \lambda \delta + \lambda \bigl(\textstyle\sum_j |c_{ij}| \|\proj_{S_\ell^\perp} (z_j)\|_2 \bigr)\\
		&\leq \lambda \delta + \lambda \textstyle\sum_j |c_{ij}|\delta\\
		&\leq \lambda\delta(1+\|c_i\|_1).
		\end{align*}

		We wish to bound $\|c_i\|_1$. Note that for any other feasible point $(\hat{c}_i,\hat{e}_i)$ of $P(x_i,X_{-i},\lambda)$, by optimality we must have
		\begin{align}\label{eq:opt-val-bound}
		\|c_i\|_1 + \dfrac{\lambda}{2}\|e_i\|_2^2 \leq \|\hat{c}_i\|_1 + \dfrac{\lambda}{2}\|\hat{e}_i\|_2^2.\end{align}
		By (\ref{slack}), we have $\lambda e_i = \nu$. Therefore,
		$$\dfrac{\lambda}{2}\|e_i\|_2^2 = \dfrac{1}{2\lambda}\|\nu\|_2^2 \geq \dfrac{1}{2\lambda}\|\nu_2\|_2^2.$$
		Therefore,
		\begin{align*}
		\|\nu_2\|_2  \leq \lambda\delta(1+\|c_i\|_1)
		\leq \lambda\delta + \lambda\delta\left(\|\hat{c}_i\|_1 + \dfrac{\lambda}{2}\|\hat{e}_i\|_2^2 - \dfrac{1}{2\lambda}\|\nu_2\|_2^2\right)
		\end{align*}
		which implies 
		\begin{align}
		\|\nu_2\|_2 +\dfrac{\delta}{2}\|\nu_2\|_2^2 &\leq \lambda\delta + \lambda\delta\left(\|\hat{c}_i\|_1 + \dfrac{\lambda}{2}\|\hat{e}_i\|_2^2\right) \label{nu2_bound}.
		\end{align}

		We now wish to construct $(\hat{c}_i, \hat{e}_i)$ such that we can bound $\|\hat{c}_i\|_1 + \dfrac{\lambda}{2}\|\hat{e}_i\|_2^2$ on the right hand side of \eqref{nu2_bound}. 
		Consider the following optimization program
		\begin{equation}\label{hat_opt}
		\min_{c}~\|c\|_1 \quad \text{subject~to}~~y_i = Y_{-i}^{(\ell)}c
		\end{equation}
		and its Lagrangian dual as 
		\begin{equation}\label{hat_opt_dual}
		\max_{\nu}~\langle y_i,\nu\rangle\quad \text{subject~to}~~\|(Y_{-i}^{(\ell)})^T\nu\|_\infty \leq 1.
		\end{equation}
 
	    Note that outside of a degenerate situation, avoided by the assumption of the lemma, the optimization problem in \eqref{hat_opt} is feasible. This program is also clearly bounded. Now, since the primal has a finite optimal value, then so does the dual, the optimal values coincide, and optimal solutions to both \eqref{hat_opt} and \eqref{hat_opt_dual} exist.

	    Let $\hat{c}_i$ be a solution to \eqref{hat_opt}. Moreover, consider $\hat{e}_i = x_i - X_{-i}^{(\ell)}\hat{c}_i = z_i - Z_{-i}^{(\ell)}\hat{c}_i$ to make $(\hat{c}_i, \hat{e}_i)$ feasible for $P(x_i,X_{-i},\lambda)$. 
		Let $\hat{\nu}$ be a solution to \eqref{hat_opt_dual} of smallest norm. 
		Note that if $\nu$ is optimal in \eqref{hat_opt_dual} then so is $\proj_{S_\ell}\nu$ as this does not alter the objective nor does it violate any constraint, as $y_i$ and the columns of $Y_{-i}^{(\ell)}$ all lie in $S_\ell$. Therefore, if we take $\hat{\nu}$ to be the solution of the smallest norm, we must have $\hat{\nu} \in S_\ell$. We also know by (\ref{hat_opt_dual}) that $\hat{\nu} \in (\qil)^\circ$. Therefore,
		$$\hat{\nu} \in (\qil)^\circ\cap S_\ell.$$
		By Lemma \ref{polytope_rel_radii},
		\begin{align*}
		\|\hat{\nu}\|_2 \leq R_{S_\ell}((\qil)^\circ\cap S_\ell)
		= \frac{1}{r_{S_\ell}(\qil)}.\end{align*}
		By strong duality between \eqref{hat_opt} and \eqref{hat_opt_dual}, we have
		\begin{equation}\label{hatc_bound}
		\|\hat{c}_i\|_1 = \langle y_i,\hat{\nu}\rangle \leq \|\hat{\nu}\|_2\|y_i\|_2 \leq \frac{1}{r_{S_\ell}(\qil)}
		\end{equation}
		as we have assumed $\|y_i\|_2=1$. 
		We can also bound $\|\hat{e}_i\|_2^2$, using the fact that
		\begin{equation}\label{hate_bound}
		\|\hat{e}_i\|_2 
		= \|z_i - Z_{-i}^{(\ell)}\hat{c}_i\|_2
		\leq \|z_i\|_2 + \sum_j |\hat{c}_{ij}|\|z_j\| \leq \delta(1+\|\hat{c}_i\|_1).
		\end{equation}
		Plugging (\ref{hate_bound}) in to (\ref{nu2_bound}) we get
		\begin{align*}
		\|\nu_2\|_2 + \frac{\delta}{2}\|\nu_2\|_2^2 \leq \lambda\delta(1+\|\hat{c}_i\|_1) + \dfrac{\delta}{2}\left(\lambda\delta(1+\|\hat{c}_i\|_1)\right)^2.\numberthis\label{nu2_more}\end{align*}
		Observe that the above can be written as $f(\|\nu_2\|_2) \leq f\big(\lambda\delta(1+\|\hat{c}_i\|_1)\big)$ for $f(x) = x + \frac{\delta}{2}x^2$. Since $f$ is monotonically increasing for $x \geq 0$, we get 
		\[
		\|\nu_2\|_2 \leq \lambda\delta(1+\|\hat{c}_i\|_1) 
		\]
		which combined with \eqref{hatc_bound} establishes the claimed bound. 
		\end{proof}

\subsubsection{Final Bound on $\|\nu\|_2$}
		Combining the bounds on $\|\nu_1\|_2$ and $\|\nu_2\|_2$ we get the following lemma.
		\begin{lemma}\label{nu_norm}
		If $0<\delta< r_{S_\ell}(\qil) $, 
		\begin{align*}
		\|\nu\|_2 \leq 
		\frac{1+\lambda\delta(1+r_{S_\ell}(\qil))}{r_{S_\ell}(\qil)-\delta}.
		\end{align*}
		\end{lemma}
		\begin{proof}
		For simplicity, let $r$ denote $r_{S_\ell}(\qil)$. Combining the results of Lemma \ref{lem:nu1-norm} and Lemma \ref{lem:nu2-norm} we get  
		\begin{align*}
				\|\nu\|_2 
				&\leq \|\nu_1\|_2  + \|\nu_2\|_2 \\
				&\leq \dfrac{1+\lambda\delta^2\left(\frac{1}{r}+1\right)}{r-\delta} + \lambda\delta\left(\frac{1}{r}+1\right) \\
				&\leq \frac{1}{r(r-\delta)} \left( r + \lambda\delta^2(r+1) +\lambda\delta (r+1)(r-\delta) \right)  \\
				& = \frac{1+\lambda\delta(1+r)}{r-\delta} \qedhere
		\end{align*}
		\end{proof}

	\subsection{Proof of Lemma \ref{intermediate_det}: Establishing the Subspace Detection Property}

		\begin{proof}[Proof of Lemma \ref{intermediate_det}]By Lemma \ref{suff_det_cond}, it suffices to show that for all $i$ and $\ell$ such that $y_i \in S_\ell$, we have
		\begin{equation}\label{temp_suff}
		(\mu(\xl)+\delta)\|\nu\|_2 < 1.\end{equation}
		For simplicity, let $r$ denote $r_{S_\ell}(\qil)$ and let $\mu$ denote $\mu(\xl)$. By Lemma \ref{nu_norm}, we have
		\begin{align*}
		(\mu+\delta)\|\nu\|_2
		\leq 	(\mu+\delta) \frac{1+\lambda\delta(1+r)}{r-\delta}.
		\end{align*}
		Therefore, \eqref{temp_suff} holds if
		\begin{align*}
			(\mu+\delta) \frac{1+\lambda\delta(1+r)}{r-\delta} < 1,
		\end{align*}
which can be equivalently stated as 
		\begin{align*}
			\lambda\delta(1+r)(\mu+\delta)  < r-\mu-2\delta.
		\end{align*}
	
		Since we assumed all of the original points in $\mathcal{Y}$ are on the unit sphere, we have $r = r(\qil) \leq 1$. Therefore, the following inequality ensures the above inequality, 
		\begin{align*}
		2\lambda\delta < \frac{r-\mu-2\delta}{\mu+\delta}.
		\end{align*}
		Since the right-hand side decreases as $r=r_{S_\ell}(\qil)$ decreases, it suffices to satisfy this condition for $i$ minimizing $r_{S_\ell}(\qil)$, which is how we defined $r_\ell$. This finishes the proof. 
		\end{proof}

	\subsection{Proof of Lemma \ref{lambda_lower}: Non-trivial Solution}
		Let us first state and prove the following result from convex geometry. 
	    \begin{lemma}\label{lem:dual-Qnorm-inradius}
	    For a set of points $\mathcal{Y} = \{y_1,\ldots, y_D\}\subset \mathbb{R}^d$, denote the symmetrized convex hull of $\mathcal{Y}$ by $\mQ$. Suppose $y_1,\ldots, y_D$ span the whole space. 
	    Then for any vector $u \in \mathbb{R}^d$ with $\|u\|_2=1$ and any $i\in \{1,\ldots,D\}$, we have
	    $$\max_{i=1,\ldots,D}~|\langle y_i,u\rangle| \geq r(\mQ).$$
	    Moreover, there exists a unit vector $u$ that achieves the above with equality. 
	    \end{lemma}
	    \begin{proof}
	    Let us rephrase the claim of the lemma. Since $\mQ$ is convex, full-dimensional, and symmetric, it has the origin in its relative interior. Therefore, the symmetric gauge function associated with $\mQ$, given by 
	    \[
	    \|w\|_\mQ \coloneqq \inf \{\gamma>0 :~ w \in \gamma\mQ \}
	    \]
	    is a norm. Consider the dual norm $\|v\|_\mQ^* \coloneqq \sup \{ \langle v,w \rangle:~ \|w\|_\mQ \leq 1 \}$. It is known that $\|v\|_\mQ^* = \|v\|_{\mQ^\circ}$ for all $v$. Moreover, as $\mQ$ is a polytope, it is easy to show that 
	    \[
	    \|v\|_\mQ^* = \max_{i=1,\ldots, D}~ |\langle y_i , v\rangle |. 
	    \]
	    Therefore, we shall prove 
	    \begin{align}
	    r(\mQ) = \min_{u\neq 0} \frac{\|u\|_\mQ^*}{\|u\|_2}
	    \end{align}
	    which establishes both of the claims of the lemma. Observe that 
	    \begin{align*}
	    \frac{1}{r(\mQ)} 
	    = R(\mQ^\circ) 
	    = \max_{w\in \partial (\mQ^\circ)} \|w\|_2
	    = \max_{w\neq 0}~\frac{\|w\|_2}{\|w\|_{\mQ^\circ}}
	    = \max_{w\neq 0}~\frac{\|w\|_2}{\|w\|_{\mQ}^*}
	    \end{align*}
	    which is the desired statement. 
	    \end{proof}

		The next lemma was first stated in Section 5.4~of \cite{Wang2016}, but we added an auxiliary lemma required in the proof, which is stated as Lemma \ref{lem:dual-Qnorm-inradius} above.
		\begin{lemma}\label{lambda_lower2}If for all $\ell\in\{1,\ldots,L\}$,
	    \begin{align}\label{eq:lem-lambda_lower2}
	    \lambda > \frac{1}{r_\ell-2\delta-\delta^2} >0	
	    \end{align}
	    then, the solution $(c,e)$ to $P(x_i,X_{-i},\lambda)$ satisfies $c \neq 0$.\end{lemma}
 	    \begin{proof}
		Suppose $c=0$ and $(c,e)$ is an optimal solution to $P(x_i,X_{-i},\lambda)$. Feasibility of $(c,e)$ implies $e = x_i - X_{-i}c = x_i$. 
		As we argued in the beginning of Section~\ref{sec:dual-cert-construction}, there exists $\nu$ that attains the optimum for $D(x_i,X_{-i},\lambda)$. We will show that for $\lambda$ larger than a certain threshold, $\nu$ does not satisfy $\|X_{-i}^T\nu \|_\infty \leq 1$, hence arriving at a contradiction. This establishes that when $\lambda$ is larger than such threshold all the optimal solutions $(c,e)$ to $P(x_i,X_{-i},\lambda)$ satisfy~$c\neq 0$. 

	    Recall that by complementary slackness we have $\nu = \lambda e = \lambda x_i$. Therefore, 
	    \begin{align*}
	    \|X_{-i}^T\nu\|_\infty &= \lambda\max_{j\neq i} |\langle x_j,x_i\rangle|.\end{align*}
	    Then, for any $j \neq i$ we have
	    \begin{align*}
	    \|X_{-i}^T\nu\|_\infty &\geq \lambda|\langle x_j,x_i\rangle |\\
	    &= \lambda |\langle y_j,y_i\rangle + \langle y_j,z_i\rangle + \langle z_j,y_i\rangle + \langle z_j,z_i\rangle|\\
	    &\geq \lambda(|\langle y_j,y_i\rangle| - |\langle y_j,z_i\rangle| - |\langle z_j,y_i\rangle| - |\langle z_j,z_i\rangle|)\\
	    &\geq \lambda(|\langle y_j,y_i\rangle| - 2\delta - \delta^2).\end{align*}
	    Therefore, 
	    \begin{align}\label{eq:nu-feas-ineq}
	    \|X_{-i}^T\nu\|_\infty \geq \lambda(\|Y_{-i}^Ty_i\|_\infty - 2\delta-\delta^2)
	    \geq \lambda(\|(Y_{-i}^{(\ell)})^Ty_i\|_\infty - 2\delta-\delta^2).
	    \end{align}
		Restricting Lemma \ref{lem:dual-Qnorm-inradius} to a subspace $S_\ell$, we get $\|(Y_{-i}^{(\ell)})^Ty_i\|_\infty \geq r_{S_\ell}(\qil)$. Therefore, from \eqref{eq:nu-feas-ineq} we get 
	    \begin{align*}
	    \|X_{-i}^T\nu\|_\infty &\geq \lambda(r_{S_\ell}(\qil) - 2\delta-\delta^2).\end{align*}
	    Taking the minimum of $r_{S_\ell}(\qil)$ over all $y_i$ lying in $S_\ell$, we have 
	    $$\|X_{-i}^T\nu\|_\infty \geq \lambda(r_\ell - 2\delta-\delta^2).$$
		Assuming $\lambda (r_\ell-2\delta-\delta^2) >1$, as in the statement of the lemma, implies $\|X_{-i}^T\nu\|_\infty > 1$ which contradicts the feasibility of $\nu$ in $D(x_i,X_{-i},\lambda)$. Hence, $(0,x_i)$ cannot be an optimal solution for $P(x_i,X_{-i},\lambda)$, which finishes the proof. 
	    \end{proof}

Note that in general $r_\ell-2\delta-\delta^2 \not>0 $ and we require further assumptions to ensure this for \eqref{eq:lem-lambda_lower2}. The conditions of Theorem \ref{det_thm} are sufficient for this quantity to be positive. 

	\subsection{Proof of Theorem \ref{det_thm}}\label{app:pf_det_thm}

		Recall the following defintions: 
		\begin{align*}
			r \coloneqq \min_{\ell=1,\ldots,L} ~r_\ell,
			\quad\text{and}\quad
			\mu \coloneqq \max_{\ell=1,\ldots,L}~ \mu_\ell.
		\end{align*}

		\begin{proof}[Proof of Theorem~\ref{det_thm}] By Lemma \ref{intermediate_det} and Lemma \ref{lambda_lower2}, the desired condition will hold if for all $\ell$,
		\begin{equation}\label{cond_d1}
		2\lambda\delta < \dfrac{r_\ell-\mu_\ell-2\delta}{\mu_\ell+\delta}\end{equation}
		and
		\begin{equation}\label{cond_d2}
		\lambda > \frac{1}{r_\ell-2\delta-\delta^2} >0.\end{equation}
		As the right-hand side of \eqref{cond_d1} is increasing in $r_\ell$ and decreasing in $\mu_\ell$, we get a sufficient condition by replacing $r_\ell$ by $r$ and $\mu_\ell$ by $\mu$. A similar argument holds for \eqref{cond_d2}. Hence, we would like 
		\begin{align}\label{eq:lambda-tight}
		0< f(\delta)\coloneqq \frac{1}{r-2\delta-\delta^2}
		< \lambda 
		< \frac{r-\mu-2\delta}{2\delta(\mu+\delta)} \eqqcolon h(\delta)
		\end{align}
		It is easy to see that whenever $f(\delta)>0$, it is increasing in $\delta$. Moreover, simple calculations show that $h$ is non-increasing in $\delta$ when $0\leq 2\delta \leq (r-\mu)+ \sqrt{r^2-\mu^2}$. 

		We claim that all of the following inequalities hold: 
		\begin{align}\label{eq:list-ineq}
		0 
		< 	\frac{1}{r-2\delta-\delta^2}
		<  \frac{1}{r-3\delta}
		<\frac{5}{2r + 3\mu}
		< \frac{15}{2r + 8\mu}
		<\frac{r-\mu-2\delta}{2\delta(\mu+\delta)}.
		\end{align}
		Recall the assumption of the Theorem in \eqref{det_crit} as 
		\[
		\delta < \dfrac{r-\mu}{5}		. 
		\] 
		It implies $\delta < \frac{r}{5}$ which, together with $0\leq \delta <r<1$, establishes ${0 < r-3\delta < r-2\delta-\delta^2}$. The third and fifth (last) inequalities in \eqref{eq:list-ineq} are derived by simply using \eqref{det_crit}. The fourth inequality is straightforward. Therefore, if $\lambda$ lies in the range specified in \eqref{lambda_crit}, both \eqref{cond_d1} and \eqref{cond_d2} will be satisfied. 
		\end{proof}

\section{Proofs for Random Models and Missing Data}
	\subsection{Proof of Theorem \ref{rand_thm}}\label{app:pf_rand_thm}

In bounding $r$, we are dealing with the true samples. Therefore, results from \cite{soltanolkotabi2012} can be used. In bounding $\mu$, note that it is defined via the corrupted samples. However, if each sample has a uniform marginal distribution over the unit sphere, we can use a spherical cap argument and a union bound to bound $\mu$. 

	    \begin{proof}[Proof of Theorem~\ref{rand_thm}]
		We want to show that with high probability, the condition of Theorem~\ref{det_thm} holds. In particular, we want to show that with high probability, we have
	    \begin{equation*}
	    \delta < \frac{r - \mu}{5}. 
	    \end{equation*}
	    To bound the right-hand side from below, we need a lower bound on $r$ and an upper bound on $\mu$. The following lemma was given in page 2229 of \cite{soltanolkotabi2012}.

	    \begin{lemma}\label{r_bound}
	    Under the random model, 
	    \begin{align*}
	    \prob\bigg(\forall (i,\ell)\text{ s.t. }y_i \in S_\ell:~r(\qil) \geq \dfrac{\const(\kappa_\ell)\sqrt{\log(\kappa_\ell)}}{\sqrt{2d_\ell}}\bigg)
	    \geq 1-\sum_{\ell=1}^L N_\ell e^{-\sqrt{\kappa_\ell}d_\ell}.\end{align*}\end{lemma}

Here, $\const(\kappa)$ is a constant depending only on $\kappa$. For $\kappa > 1$, $\const(\kappa) > 0$ and there is a $\kappa_0$ for all $\kappa \geq \kappa_0$, we can take $\const(\kappa) = \frac{1}{\sqrt{8}}$ \cite{soltanolkotabi2012}. 	    

	    Note that this lemma is equivalent to saying that
	    $$\prob\bigg(r \geq \min_\ell \frac{\const(\kappa_\ell)\sqrt{\log(\kappa_\ell)}}{\sqrt{2d_\ell}}\bigg)
	    \geq 1-\sum_{\ell=1}^L N_\ell e^{-\sqrt{\kappa_\ell}d_\ell}.$$

	    Next, recall that we define $\mu_\ell$ by
	    \begin{equation}
	    \mu_\ell = \max_{ \substack{y \in \mathcal{Y}\back\mathcal{Y}^{(\ell)} \\ 1 \leq i \leq N_\ell } } |\langle v_i^{(\ell)},y\rangle|.\end{equation}
	    Here, $v_i^{(\ell)}$ is a unit vector by definition. Fix $y \in \mathcal{Y}\back\mathcal{Y}^{(\ell)}$ and some vector $y_i$ drawn from $S_\ell$. Then the dual direction $v_i$ is a unit vector depending only on the samples corresponding to $S_\ell$. In particular, $y$ is independent from these samples, therefore $y$ and $v_i$ are independent. We also know that $y$ has marginal distribution that is uniform on the unit sphere. This follows from the fact that the subspace $S_j$ from which $y$ is drawn is selected uniformly among all $d_j$-dimensional subspaces, and $y$ is selected uniformly at random from the unit ball in $S_j$. We can therefore use the following consequence of well-known results concerning spherical cap densities.
	    \begin{lemma}[\cite{ball1997elementary}]\label{spher_cap}
	    Let $y$ be a vector uniformly distributed on the unit sphere $S^{n-1}$ and let $a$ be a fixed unit vector on $S^{n-1}$. Then for any $\epsilon > 0$,
	    $$\prob\bigg(|\langle a,y\rangle | > \epsilon\bigg) \leq 2\exp\left(-\dfrac{n\epsilon^2}{2}\right).$$\end{lemma}
	    Applying Lemma \ref{spher_cap} with $\epsilon = \sqrt{6\log N/n}$, $a = v_i^{(\ell)}$, and the $y$ above, we get:
	    \begin{equation}\label{spher_cap2}
	    \prob\bigg(|\langle v_i^{(\ell)},y\rangle | > \sqrt{ \frac{6\log N}{n} }\bigg) \leq\dfrac{2}{N^3}.\end{equation}
	   	Taking a union bound of (\ref{spher_cap2}) over all such $y$ and pairs $(i,\ell)$ with $y_i \in S_\ell$, we derive the following lemma.
	   			\begin{lemma}\label{u_bound}
		$$\prob\left(\mu_\ell \leq \sqrt{ \frac{6\log N}{n} }\text{ for all }\ell\right) \geq 1-\frac{2}{N}.$$\end{lemma}
		In particular, this implies
		$$\prob\left(\mu \leq \sqrt{ \frac{6\log N}{n} }\right) \geq 1-\frac{2}{N}.$$

	    By the union bound, Lemmas \ref{r_bound} and \ref{u_bound} show that with probability at least $1-\frac{2}{N}-\sum_{\ell=1}^L N_\ell e^{-\sqrt{\kappa_\ell}d_\ell}$, we have that for all $\ell$,
	    \begin{gather*}
	    r_\ell \geq \dfrac{\const(\kappa_\ell)\sqrt{\log(\kappa_\ell)}}{\sqrt{2d_\ell}}
	    ~~,~~
	    \mu_\ell \leq \sqrt{ \dfrac{6\log N}{n}}.
	    \end{gather*}
	    Assume that for all $\ell$,
	    \begin{equation}\label{eq:assump-d}
	    d_\ell \leq \dfrac{\const(\kappa_\ell)^2\log(\kappa_\ell)}{48\log N}n.
	    \end{equation}
	    So, with probability at least $1-\frac{2}{N}-\sum_{\ell=1}^L N_\ell e^{-\sqrt{\kappa_\ell}d_\ell}$,
	    \begin{align}\label{eq:r-ell-mu-ell}
	    r_\ell &\geq \dfrac{\const(\kappa_\ell)\sqrt{\log(\kappa_\ell)}}{\sqrt{2d_\ell}}
	     \geq\sqrt{ \dfrac{24\log N }{n}}
	     \geq 2\mu_\ell.
	    \end{align}
	    In particular, this implies that with the same probability
	    \begin{align}\label{eq:bound-r-mu}
	    r \geq \sqrt{\dfrac{24\log N }{n}} \geq 2\mu.
	    \end{align}
	    Therefore, with this same probability
	    \begin{align*}
	    r-\mu \geq \frac{r}{2}\end{align*}
	    and so
	    \begin{equation}\label{rand_5r}
	    \dfrac{r-\mu}{5} \geq \dfrac{\const(\kappa_\ell)}{10\sqrt{2}}\dfrac{ \sqrt{\log(\kappa_\ell)}}{\sqrt{d_\ell}}.\end{equation}
	    If we require
	    \begin{align}\label{eq:delta-rand}
	    \delta < \dfrac{\const(\kappa_\ell)}{10\sqrt{2}}\dfrac{ \sqrt{\log(\kappa_\ell)}}{\sqrt{d_\ell}},
	    \end{align}
	    then with probability at least $1-\frac{2}{N}-\sum_{\ell=1}^L N_\ell e^{-\sqrt{\kappa_\ell}d_\ell}$, the geometric separation condition and therefore the subspace detection property will hold. Taking $c_1 = \frac{1}{48}, c_2 = \frac{1}{10\sqrt{2}}$, we get conditions (\ref{rand_dl_cond}) and (\ref{rand_delta_cond}).

	    By the same reasoning as in the proof of Theorem \ref{det_thm}, we can derive the same interval of $\lambda$ for which the subspace detection property will hold and the output of LS-SSC will be non-trivial. We showed in the proof of Theorem \ref{det_thm} that this will occur as long as \eqref{eq:list-ineq} holds; i.e., 
	    \begin{equation}\label{random_lambda}
\frac{1}{r-3\delta}  
		< \lambda 
		< \frac{r-\mu-2\delta}{2\delta(\mu+\delta)}.
	    \end{equation}
	    Above, we showed that with probability at least $1-\frac{2}{N}-\sum_{\ell=1}^L N_\ell e^{-\sqrt{\kappa_\ell}d_\ell}$, under the assumption in \eqref{eq:assump-d} on $d_\ell$, we have \eqref{eq:bound-r-mu} and \eqref{eq:delta-rand}. Plugging these bounds in to the above equation, we find that with the same probability above, the subspace detection property with parameter $\lambda$ will hold and the output of LS-SSC will be non-trivial as long as
	    \begin{equation}\label{random_lambda_2}
	    \frac{10}{7}\sqrt{\dfrac{n}{24 \log N}} < \lambda < \frac{20}{3}\sqrt{\dfrac{n}{24 \log N}}.\end{equation}

	    This is a non-empty interval of $\lambda$ for which LS-SSC has the subspace detection property and has non-trivial output, with the given probability above.
	    \end{proof}

	\subsection{Proof of Theorem \ref{scmd_rand}}\label{app:pf_scmd_rand}
		Let $c_1$ and $c_3$ be the same constants as in Theorem \ref{rand_thm}. To prove this theorem, we will use the previously mentioned result about the effect of projection on the norm, in Lemma~\ref{proj_lem}.

		We will now prove Theorem \ref{scmd_rand}. Note that the required condition on $d_\ell$ in (\ref{missing_d_cond}) is the same as in Theorem \ref{rand_thm}. Therefore, to prove Theorem \ref{scmd_rand}, it suffices to find conditions on the number of missing entries such that the required bound on $\delta$ in Theorem \ref{rand_thm} holds. It suffices to show that for each column $z$ of $Z$, we have
	    $$\|z\|_2 < c_2\const(\kappa_\ell)\sqrt{\dfrac{\log(\kappa_\ell)}{d_\ell}}.$$
	    For each $\ell$, define
	    $$M_\ell \coloneqq \dfrac{c_2^2(1-\epsilon)^2\const(\kappa_\ell)^2\log(\kappa_\ell)n}{d_\ell}.$$
	    Suppose that \eqref{missing_d_cond} holds; i.e., 
	    \begin{align*}
	    d_\ell \leq \dfrac{c_1\const(\kappa_\ell)^2 \log(\kappa_\ell)n}{\log N}.
	    \end{align*}

	    Fix some column $y$ in $Y$ coming from $S_\ell$. Let $m_\ell$ denote the number of its missing entries. Since $S_\ell$ is chosen uniformly at random from all $d_\ell$-dimensional subspaces and $y$ is chosen uniformly on $S^{n-1}\cap S_\ell$, $y$ has a marginal distribution that is uniform on the unit sphere. Let $\{i_1,\ldots, i_{m_\ell}\}$ denote the locations of the missing entries of $y$ in the corresponding observed sample $x$. 

		Let $U$ be the span of $\{e_{i_1},\ldots, e_{i_{m_\ell}}\}$. Since $U$ is chosen independently from $Y$, we can consider $U$ as fixed with respect to $y$. Then $z = -\proj_U(y)$ and we are interested in an upper bound on $\|z\|_2$, with high probability. Adding nonzero entries to $z$ only makes the norm increase. Therefore, we assume that $z$ has $M_\ell$ nonzero entries instead of $m_\ell$. Then, by Lemma \ref{proj_lem}, for any $\epsilon > 0$, with probability at least $1-2\exp(-\epsilon^2M_\ell/4)$ we have
		\begin{align*}
		\|z\|_2 
		\leq (1-\epsilon)^{-1}\sqrt{\dfrac{M_\ell}{n}}
		= c_2\const(\kappa_\ell)\sqrt{\dfrac{\log(\kappa_\ell)}{d_\ell}}.\end{align*}
		Taking a union bound, this holds for all columns $z$ of $Z$ with probability at least
		$$1-2\sum_{\ell=1}^L N_\ell e^{-M_\ell/16}.$$
		Therefore, the conditions of Theorem \ref{rand_thm} hold with at least this probability. Taking $\epsilon = 0.5$, letting $c_3 = c_2^2/4$, and taking a union bound with the probability of success for Theorem \ref{rand_thm}, we derive Theorem \ref{scmd_rand}.

\end{document}